\documentclass[sigconf]{acmart}
\usepackage{xcolor}
\usepackage{latexsym}
\usepackage[shortlabels]{enumitem}
\usepackage{amsfonts}
\usepackage{amsmath}
\usepackage{mathtools}
\usepackage{tabularray}
\usepackage{xcolor}
\usepackage{tcolorbox}
\usepackage{soul} % 加载 soul 宏包
\usepackage{hyperref}
\usepackage{amsthm}
\newtheorem{theorem}{Theorem}
\newtheorem{lemma}{Lemma}
\newcommand{\model}{\textsc{CTkvr}}
\newcommand{\QCIVF}{\textsc{qcIVF}}
\newtheorem{observation}{{\bf Obs.}}
\newtheorem{insight}{{\bf Insight}}
\usepackage[linesnumbered,ruled,vlined]{algorithm2e} % 在导言区加载 algorithm2e 包

\AtBeginDocument{%
  }

\setcopyright{acmlicensed}
\copyrightyear{2018}
\acmYear{2018}
\acmDOI{XXXXXXX.XXXXXXX}
\acmConference[Conference acronym 'XX]{Make sure to enter the correct
  conference title from your rights confirmation email}{June 03--05,
  2018}{Woodstock, NY}
\acmISBN{978-1-4503-XXXX-X/2018/06}

\begin{document}

\title{\model{}: Efficient KV Cache Retrieval for Long-Context LLMs via Centroid-then-Token Indexing}

% \affiliation{%

\author{Kuan Lu}
\authornotemark[1]
\affiliation{%
  \institution{Zhejiang University}
  \country{China}
}

\author{Shuhang Lin}
\authornotemark[1]
\affiliation{%
  \institution{Rutgers University}
  \country{USA}
}

\author{Sai Wu}
\affiliation{%
  \institution{Zhejiang University}
  \country{China}
}

\author{Yichen Yao}
\affiliation{%
  \institution{INFLY Tech}
  \country{China}
}

\author{Junhan Yang}
\affiliation{%
  \institution{INFLY Tech}
  \country{China}
}

\author{Huan Li}
\authornotemark[2]
\affiliation{%
  \institution{Zhejiang University}
  \country{China}
}

\author{Wei Chu}
\affiliation{%
  \institution{INFLY Tech}
  \country{China}
}

\author{Xu Yinghui}
\affiliation{%
  \institution{INFLY Tech}
  \country{China}
}

\author{Yuan Qi}
\affiliation{%
  \institution{INFLY Tech}
  \country{China}
}

\author{Gang Chen}
\affiliation{%
  \institution{Zhejiang University}
  \country{China}
}
\renewcommand{\shortauthors}{Trovato et al.}

\definecolor{commentcolor}{RGB}{0,114,189}
\SetKwComment{Comment}{\textcolor{commentcolor}{$\triangleright$} }{}
\newcommand{\commentstyle}[1]{\textcolor{commentcolor}{\textrm{\textit{#1}}}}

\definecolor{redcommentcolor}{RGB}{161,0,2} % 红色
\SetKwComment{BigComment}{\textcolor{redcommentcolor}{$\bullet$} }{}
\newcommand{\bigcommentstyle}[1]{\textcolor{redcommentcolor}{\textrm{\textbf{#1}}}}

\definecolor{lightblue}{RGB}{173, 216, 230} % 淡蓝色的 RGB 值
\definecolor{mygray}{RGB}{210, 210, 210}
% 设置高亮颜色为淡蓝色
\sethlcolor{lightblue}


\received{20 February 2007}
\received[revised]{12 March 2009}
\received[accepted]{5 June 2009}

\begin{abstract}
Large language models (LLMs) are increasingly applied in long-context scenarios such as multi-turn conversations. 
However, long contexts pose significant challenges for inference efficiency, including high memory overhead from Key-Value (KV) cache and increased latency due to excessive memory accesses. 
Recent methods for dynamic KV selection struggle with trade-offs: block-level indexing degrades accuracy by retrieving irrelevant KV entries, while token-level indexing incurs high latency from inefficient retrieval mechanisms.
% Recent studies have explored dynamic KV selection for attention computation, but suffer from trade-offs:
% block-grained indexing yields performance degradation by recalling massive irrelevant KV entries, while token-grained indexing incurs high latency due to inefficient KV retrieval.
In this paper, we propose \model{}, a novel centroid-then-token KV retrieval scheme that addresses these limitations.
\model{} leverages a key observation: query vectors adjacent in position exhibit high similarity after Rotary Position Embedding (RoPE) and share most of their top-$k$ KV cache entries.
% address these limitations by leveraging a key observation: query vectors adjacent in position exhibit high similarity after Rotary Position Embedding (RoPE) and therefore share most of their top-$k$ KV cache entries.
Based on this insight, \model{} employs a two-stage retrieval strategy: lightweight centroids are precomputed during prefilling for centroid-grained indexing, followed by token-level refinement for precise KV retrieval. This approach balances retrieval efficiency and accuracy.
To further enhance performance, we implement an optimized system for indexing construction and search using CPU-GPU co-execution.
Experimentally, \model{} achieves superior performance across multiple benchmarks with less than 1\% accuracy degradation. 
Meanwhile, \model{} delivers $3\times$ and $4\times$ throughput speedups on \texttt{Llama-3-8B} and \texttt{Yi-9B} at 96K context length across diverse GPU hardware.
% The code is available at https://anonymous.4open.science/r/CTKVR.
\end{abstract}
\maketitle
\section{Introduction}
\label{sec:intro}

Large language models (LLMs) with long context windows, such as GPT \citep{openai2024gpt4technicalreport}, Llama \citep{grattafiori2024llama3herdmodels}, and Gemini \citep{geminiteam2024geminifamilyhighlycapable}, have showcased remarkable scalability of handling contexts of up to 1M tokens. 
This capability empowers LLMs to tackle complex tasks like multi-document question answering (QA) and information retrieval \citep{bai2024longbench}, advancing applications such as chatbots \citep{chatbot} and search engines \citep{xiong2024searchengineservicesmeet}.
Despite their potential, efficiently serving long-context LLMs poses significant challenges, particularly with the management of the Key-Value (KV) cache --- which stores intermediate Keys and Values activations to avoid re-computation.
% As storage overhead of KV cache scales linearly with sequence length and the requirement for each KV cache access for token generation, the limited memory of current GPUs can only support a small batch size, resulting in low throughput.
Since generating each token requires accessing each KV pair, most methods store full KV cache on the GPU to minimize latency.
However, this storage budget scales linearly with the sequence length. Consequently, in long-context scenarios, as the batch size grows, the GPU's VRAM is rapidly consumed, creating a bottleneck that restricts throughput.

\begin{figure}[t]
    \centering
    \includegraphics[width=\linewidth]{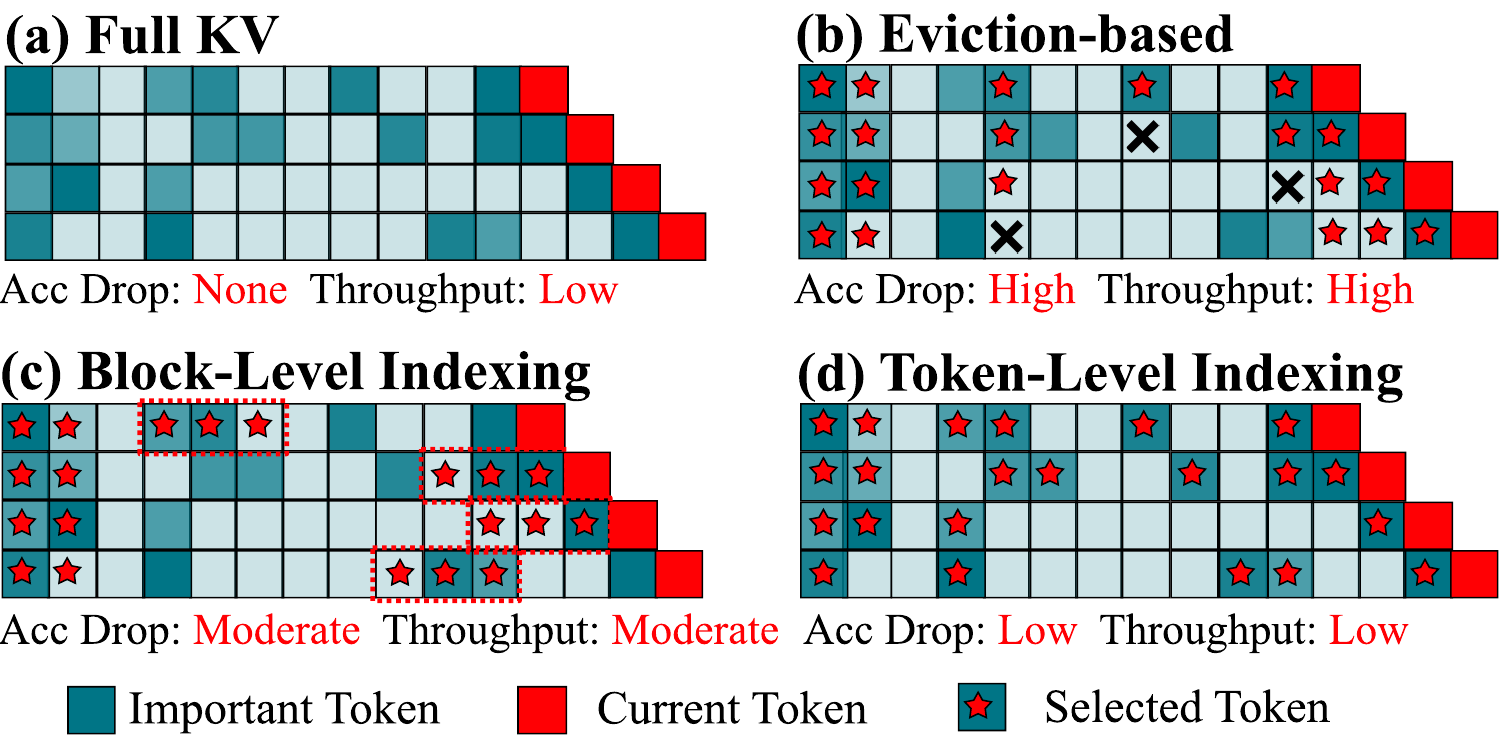}
    \caption{Illustration of KV cache compression methods and their comparison based on accuracy drop and throughput.}
    \label{fig:methods}
\end{figure}
% \textcolor{blue}{\textbf{[New version of Paragraph 2]} 
To tackle this issue, recent methods such as KV cache eviction and sparse attention with block- or token-level indexing have been extensively studied (see Figure~\ref{fig:methods}(b)-(d)). 
However, these methods face several limitations to varying degrees: (i) reduced accuracy compared to using the full KV cache (Figure~\ref{fig:methods}(a)), (ii) significant latency overhead in decoding and prefilling. 
% As shown in \autoref{fig:methods}, existing methods, such as KV cache eviction and sparse attention with block-level or token-level indexing, have been explored but fail to meet all requirements.
\textbf{\emph{KV cache eviction}} methods \citep{h2o, li2024snapkv} reduce memory usage by discarding KV pairs based on specific policies demonstrated in Figure~\ref{fig:methods}(b). 
While effective in reducing memory use, this method results in information loss, leading to accuracy decline, particularly in tasks requiring long outputs \citep{bai2024longwriter} or multi-turn conversations \citep{munkhdalai2024leavecontextbehindefficient}.
\textbf{\emph{Sparse attention methods}}~\cite{hooper2024squeezed,xiao2024infllm} reduce GPU memory footprint by offloading most KV caches to CPU and dynamically loading to the GPU the KV entries related to the current token for attention computation.
They can be further divided according to their indexing granularity:
\emph{Block-level indexing} \citep{sun2024shadowkv, tang2024quest} summarizes cached Keys into block-level vectors during prefilling and retrieves blocks\footnote{{Each Key block, which typically consists of 8 or 16 Key vectors~\cite{sun2024shadowkv}, is stored contiguously in memory.}} based on similarity during decoding shown in Figure~\ref{fig:methods}(c). While efficient, this can introduce irrelevant KV entries within blocks, reducing accuracy.
\emph{Token-level indexing} \citep{liu2024retrievalattentionacceleratinglongcontextllm, wu2024tokenselectefficientlongcontextinference} dynamically retrieves top-$k$ similar KV entries using methods like \emph{Approximate Nearest Neighbor} (ANN) in Figure~\ref{fig:methods}(d). Although this achieves higher accuracy, it imposes significant computational overhead, with latency increasing by up to $3\times$ during prefilling and $5\times$ during decoding \citep{liu2024retrievalattentionacceleratinglongcontextllm}. Pre-built ANN indexes also struggle in scenarios with long outputs or multi-turn conversations, further impairing performance.

In summary, existing KV compression methods continue to face trade-offs in accuracy, efficiency, and scalability for long-context scenarios. To address these limitations, an effective LLM serving system should: \textit{(1) maintain accuracy close to that of the full KV cache, (2) minimize prefilling and decoding latency, and (3) reduce GPU memory usage}. Achieving these goals is crucial, as emphasized in prior research \citep{sun2024shadowkv,magicpig}.

% % 这里稍微有点牵强，但是目前没想好怎么顺下来。【感觉可以直接删了这段，逻辑依旧通顺】
% To better meet the stated three requirements, two challenges arise: (i) How to leverage the existing information during prefilling to retrieve relevant KV cache with the speed of block-level indexing and the accuracy of token-level indexing. (ii) How to fully leveraging the hardware resources including GPU, CPU, and PCIe bandwidth, to reduce GPU VRAM overhead.

Fortunately, we observe that, after applying \emph{Rotary Position Embedding} (RoPE \citep{su2023roformerenhancedtransformerrotary}) in attention computations, query vectors adjacent in positions tend to exhibit high similarity and frequently share most of their top-$k$ KV entries (see Section~\ref{sec:obs}). 
Given this, we propose \textbf{\emph{Centroid-then-Token KV Retrieval}} (\model{}), a novel method that constructs a lightweight query-centroid Inverted File (\QCIVF{}) during prefilling for efficient LLM decoding. 
\QCIVF{} is computationally efficient to build and search, supporting massive centroids for higher accuracy \citep{douze2024faisslibrary}. 
Furthermore, inspired by \citet{he2024fastdecodehighthroughputgpuefficientllm}, we offload partial KV cache computation and retrieval to the CPU, allowing for larger batch sizes by utilizing DRAM. 
{To further maximize throughput, our optimization incorporates custom CUDA kernels and utilizes CUDA multi-streaming techniques.}
Our contributions are summarized as follows.
\begin{itemize}[leftmargin=*]
% [leftmargin=*,itemsep=-0.5em,topsep=0em]
    \item \textbf{\emph{Efficient and Effective Query-Centric Retrieval Algorithm}} (Section \ref{sec:algo}):
    We propose \model{}, a method designed for efficiently retrieving top-$k$ Keys from the full Key cache based on a centroid-then-token indexing. 
    \model{} performs a two-stage indexing process: (1) \emph{Query-centroid Indexing} that groups query vectors into centroids and gathers corresponding Keys for a coarse-grained search; and (2) \emph{Fine-grained Token-Level Indexing} that refines the search by retrieving the concrete top-$k$ results with high accuracy. This hierarchical indexing ensures both computational efficiency and retrieval effectiveness.

    % This token-level algorithm first index the query centroids and gather their corresponding Keys. It then performs a secondary token-level indexing to determine the final top-$k$ results. The first and second indexing separately ensures the efficiency and effectiveness.
    % This token-level algorithm ensures efficiency by first locating the query centroids with the corresponding Keys, following by a secondary token-level search guarantees high accuracy in the results.
    % Additionally, it dynamically increases query centroids during inference to address potential inaccuracies in searching long output queries.
    
    \item \textbf{\emph{System Designs for Scalability and Efficiency}} (Section \ref{sec:sys}): We optimize system performance by (1) offloading partial KV cache attention computation to the CPU, enabling larger batch sizes and longer contexts by leveraging DRAM and (2) accelerating the overall process via custom CUDA kernels and CUDA multi-streaming techniques, enabling overlap between GPU and CPU computations to maximize throughput.
    % In Section \ref{sec:sys}, we present our system design. We enhance system efficiency by (1) offloading KV cache attention computation to the CPU for larger batch size and longer context (2) accelerating the overall process through the development of custom CUDA kernels and leveraging CUDA multi-stream techniques to overlap GPU and CPU attention computation times.
    
    \item \textbf{\emph{Empirical Validation of Accuracy and Efficiency}} (Section \ref{sec:exp}): \model{} achieves superior accuracy across multiple well-recognized benchmarks, with less than 1\% accuracy degradation.
    More importantly, \model{} delivers 3$\times$ and 4$\times$ throughput speedups on LLMs like \texttt{Llama-3-8B} and \texttt{Yi-9B} across various GPU specifications with 96K context lengths.
    % In Section \ref{sec:exp}, we demonstrate the empirical evaluation results of \model{} in accuracy and efficiency. While maintaining high accuracy in various benchmark including RULER \citep{hsieh2024ruler}, LongBench \citep{bai2024longbench}, Needles-in-a-haystack \citep{needle} with less than 1\% degration, \model{} improves decoding throughput up to 3$\times$ and 4$\times$ in Llama-3-8B and Yi-9B model with 96K context.
\end{itemize}

\section{Observations and Insights} \label{sec:obs}

We present insights into long-context LLM behavior that form the foundation of \model{}'s design.
% , providing a fresh perspective on optimizing query-Key mappings for efficient decoding.

\begin{figure}
    \centering
    \includegraphics[width=0.95\linewidth]{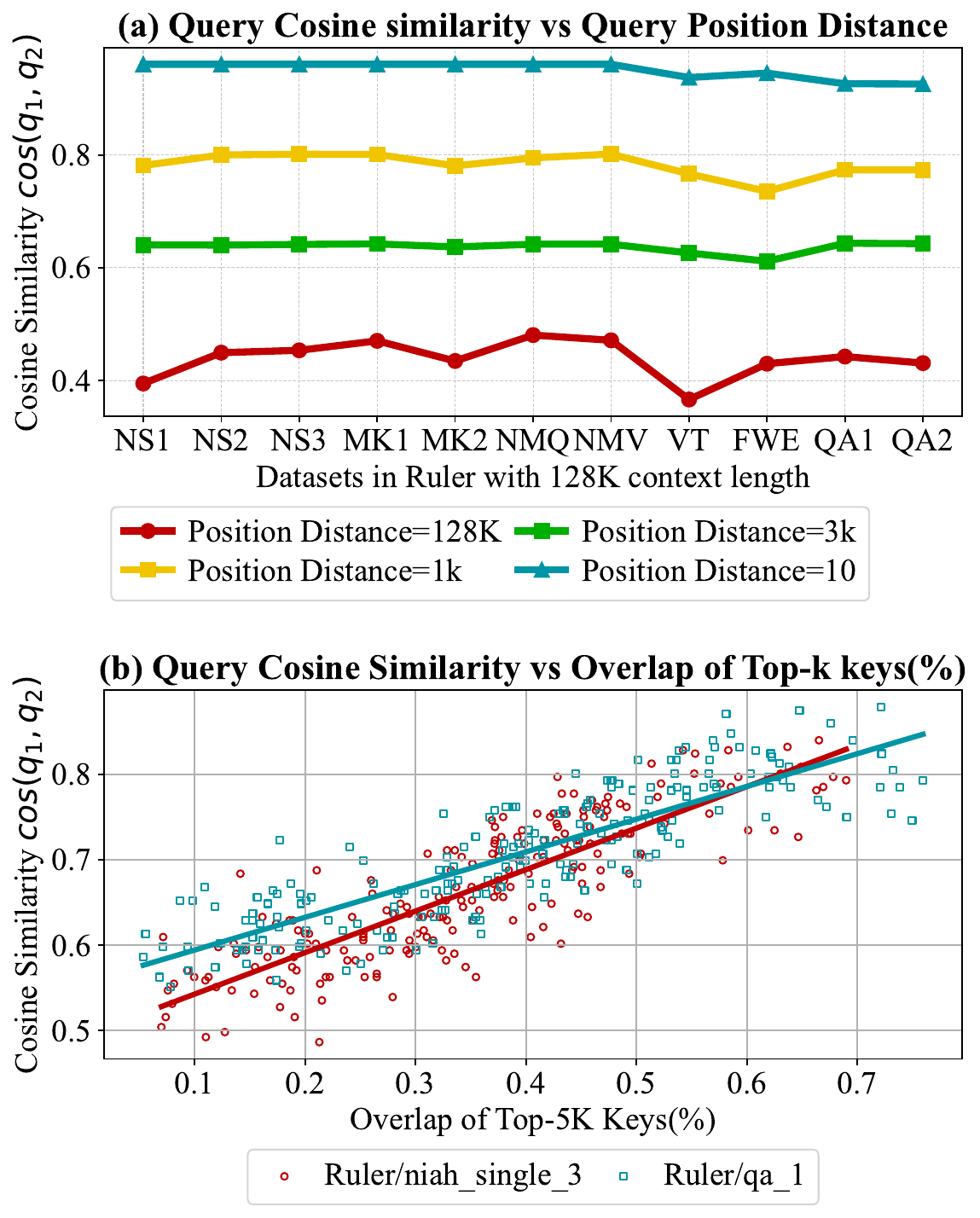}
    \caption{Analysis of query vector similarity: (a) Query distance proximity correlates positively with cosine similarity, with consistent trends across datasets. (b) The overlap of top-$k$ retrieved Keys is approximately positively correlated with query cosine similarity.}
    \label{fig:query_sim}
\end{figure}

\noindent\begin{observation}[{\bf \em Positionally Adjacent Queries are Highly Similar after RoPE}]\label{obs_1}
As shown in \autoref{fig:query_sim}(a), we tested the cosine similarity of query vectors at varying position distances across datasets in \textsf{Ruler} \citep{hsieh2024ruler}, a long-context benchmark with 13 datasets spanning retrieval, QA, multi-hop tracking, and aggregation tasks. 
The results reveal a striking pattern: positionally close query vectors exhibit high similarity.
For instance, even at a position distance within 1K, the cosine similarity consistently exceeds 0.8 in most datasets.
% demonstrate that positionally close query vectors exhibit high cosine similarity. In most dataset, it maintains a similarity score of 0.8 even at a distance of 1K.
\end{observation}

\noindent\begin{observation}[{\bf \em Similar Vectors Retrieve Overlapping Top-$k$ Keys}]\label{obs_2}
Formally, for two vectors $Q, Q' \in \mathbb{R}^{d}$ and a search space $\{K_i\}_{i=1}^N$, if their cosine similarity satisfies $\operatorname{cos}(Q, Q') > \varepsilon$, their top-$k$ retrieval sets $S$ and $S'$ exhibit significant overlap, i.e., ${|S \cap S'|}/{|S|} > p$\footnote{A formal proof has been provided in Appendix~\ref{sec:appendix_lemma}. \citet{wu2024tokenselectefficientlongcontextinference} reach a similar conclusion, which is a specialized case of the lemma here, using a different proof methodology.}.
Experimentally, we randomly sampled query vector pairs from different positions in two representative \textsf{RULER} datasets. For each pair, we compute their cosine similarity and the overlap ratio of the top-5K Keys recalled by the $QK^T$ score. These results are then used to create a scatter plot, with a fitted trend line to indicate the relationship.
As shown in \autoref{fig:query_sim}(b), there is a clear positive correlation between query similarity and the overlap of their retrieved Keys. This indicates that queries with similar embeddings tend to access highly similar Key sets, reducing the need to search the entire Key space for each query.
% a top-$k$ retrieval vector number $K$ and top-$K$ overlap percentage $p \in [0,1]$, there exist a threshold $\varepsilon$ such that when the cosine similarity of two vectors $cos(V, V') > \varepsilon$, the top-$K$ vector sets $S$ and $S'$ retrieved by V and V' would satisfy $\frac{|S\cap S'|}{|S|} > p$. This would be formally proven in appendix A. Experimentally, we selected two datasets in Ruler, randomly sampled different query vectors, and computed the their cosine similarity and overlapping number of their top-$5K$ Keys. As illustrated in \autoref{fig:query_sim}(b), we observed that the similarity and the percentage of top-$k$ Key overlap roughly exhibit a positive correlation. Meanwhile, \citep{wu2024tokenselectefficientlongcontextinference} also discover a similar conclusion which is the specialized case of lemma here, with a different method of proof.
\end{observation} 

\noindent\begin{insight}
The inefficiency of current ANN-based methods \citep{liu2024retrievalattentionacceleratinglongcontextllm} in index construction and vector retrieval stems from their need to model the full query-Key mapping across long contexts, much of which is redundant for decoding. 
Obs.~\ref{obs_2} reveals that accurately retrieving the top-$k$ Keys for a query can be achieved by limiting the search to Keys associated with queries similar to the current query.
% it is sufficient to simply model the query-Key mappings of queries similar to the current query in the search space to achieve the top-$k$ Key vectors with high-recall. 
Obs.~\ref{obs_1} provides a lightweight mechanism to identify similar queries efficiently, as positional proximity correlates strongly with query similarity.
\end{insight}

\begin{figure}[ht]
    \centering
    \includegraphics[width=\linewidth]{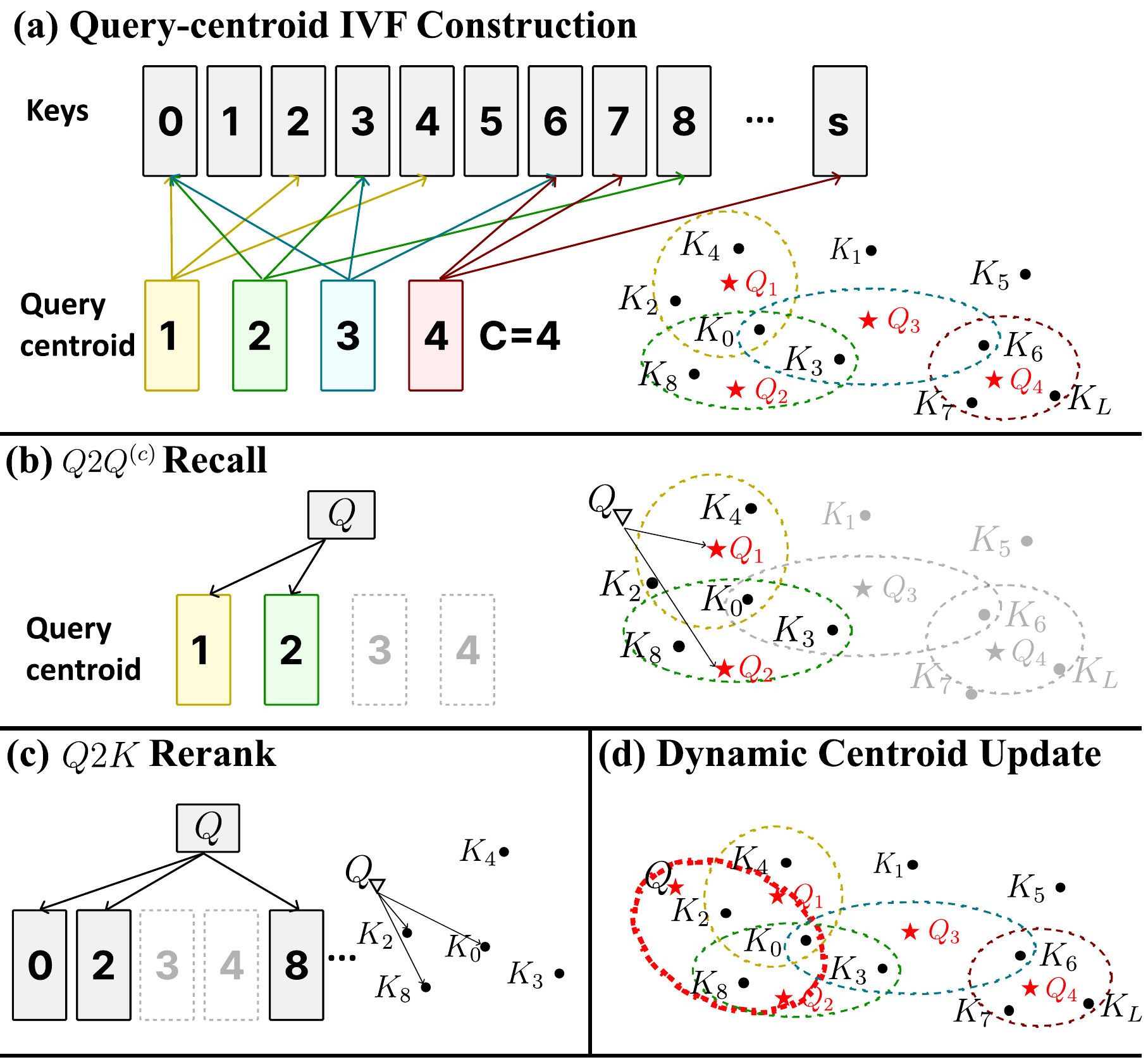}
    \caption{The main components of \model{}, the pseudocode of (a) and (b)(c)(d) are separately provided in \autoref{alg:DIDX-prefill} and \autoref{alg:DIDX-decoding}.}
    \label{fig:detail}
\end{figure}

\section{Centroid-then-Token KV Retrieval}
\label{sec:method}

Section \ref{sec:obs} highlights the potential of modeling \emph{partial} query-Key mappings to accelerate ANN search. Building on this, Section~\ref{sec:algo} outlines the query-centroid Inverted File (\QCIVF{}) construction and the top-$k$ retrieval algorithm, while Section~\ref{sec:sys} introduces a CPU-GPU co-design with CUDA optimizations for efficient LLM decoding.

\subsection{\QCIVF{} and KV Retrieval Algorithm} \label{sec:algo}

\subsubsection{\textbf{\QCIVF{} Construction}}
During the \textbf{\emph{prefilling}} procedure as shown in \autoref{fig:detail}(a) and \autoref{alg:DIDX-prefill}, we extract the last $C$ query vectors from the long context as the query centroids (\autoref{a1:l1}), leveraging Obs.~\ref{obs_1} showing that positionally adjacent queries have similar top-$k$ Keys.
For each centroid, we compute attention scores on the full Key cache to find the top-$\rho$ most similar Keys (\autoref{a1:l2}--\autoref{a1:l3}) for each KV head instead of query head\footnote{Most LLMs support \emph{Grouped-Query Attention} (GQA) that reduces computation and memory by grouping $h$ query heads into $g$ groups that share Key-Value heads~\citep{gpa} (called KV head in this paper). \model{} leverages this by modeling $g$ for \QCIVF{} (\autoref{a1:l2}--\autoref{a2:l3} in \autoref{alg:DIDX-prefill}), further reducing its storage overhead. Formally, given input $X \in \mathbb{R}^{b\times h \times s \times d}$, the output $Y \in \mathbb{R}^{b\times g\times s\times d}$ is computed as: $\mathbf{Y}[:, i, :, :] = \max\left( \mathbf{X}[:,, h_i : h_{i+1},, :, :] \right), \quad i = 0, 1, \dots, g{-}1,\quad h_i = i \cdot \frac{h}{g}$}, storing the results as the query-centroid Inverted File index \QCIVF{} (\autoref{a1:l4}).
Using this index, most of the KV cache\footnote{Following prior works, we define the batch size, sequence length, and head dimension of Key and Value as $b$, $s$, and $d$.} can be offloaded to the CPU, retaining only the initial and local KV cache on the GPU, as in \textsc{StreamingLLM} \citep{xiao2023streamingllm} (\autoref{a1:l5}--\autoref{a1:l6}).
For a 128K-length text, this approach reduces GPU storage to just 5\% of the original KV cache size. Meanwhile, this approach offers a significant speed advantage over ANN-based indexing with the negligible memory overhead in \QCIVF{}\footnote{For $b, g, C, \rho=1, 4, 512, 2560$, the memory overhead of \QCIVF{} is $1 \times g \times C \times \rho \times 4 (sizeof(int32)) = 20MB$.}.

% Moreover, \QCIVF{} could stored in either GPU or CPU considering tradeoff of storage budget and speed. If storing \QCIVF{} on the CPU, DIDX requires only 1\% of the original storage space with less than 10\% increase in end-to-end latency for token generation due to additional data transmission time over PCIe.

\begin{algorithm}[!htbp]
% \footnotesize
\fontsize{8.5}{11}\selectfont
\caption{\small{\QCIVF{} Construction at Prefilling}}
\label{alg:DIDX-prefill}
\KwIn{$K, V \in \mathbb{R}^{b \times g \times s \times d}$, $Q \in \mathbb{R}^{b \times h \times s \times d}$, number of query centroids $C$, number of recall per query $\rho$, reserved initial and local KV cache lengths $L^\text{init}$ and $L^\text{local}$, respectively.}
\BigComment{\bigcommentstyle{(a) Query-centroid IVF Construction}}
\Comment{\commentstyle{Select the last $C$ query vectors as the centroids}}
$Q^{(c)} \leftarrow Q[:, :, -C:, :]$ \label{a1:l1}\\
\Comment{\commentstyle{Attention score between query centroids and Keys}}
$\mathcal{A}_q\in \mathbb{R}^{b \times (h/g \times g) \times C \times s} \leftarrow \operatorname{Softmax}(Q^{(c)}\cdot K)$\label{a1:l2}\\ 
\Comment{\commentstyle{Compute maximum score within each GQA group}}
$\mathcal{A}_{kv}\in \mathbb{R}^{b \times g \times C \times s} \leftarrow \operatorname{max_{kv-group}}(\mathcal{A}_q)$ \label{a1:l3}\\
\Comment{\commentstyle{Select top-$\rho$ Key vectors for each query centroid's each KV head}}
% (store in GPU/CPU)
$\QCIVF{}\in \mathbb{R}^{b \times g \times C \times \rho} \leftarrow \mathcal{A}_{kv}.\operatorname{top}(\rho)$ \label{a1:l4}\\
\Comment{\commentstyle{Offload initial and local portions of KV cache to CPU}}
$K^\text{CPU}, V^\text{CPU} \leftarrow \operatorname{TruncLocalInit}(K, V, L^\text{init}, L^\text{local})$ \label{a1:l5}\\
$K^\text{GPU}, V^\text{GPU} \leftarrow \operatorname{GetLocalInit}(K, V, L^\text{init}, L^\text{local})$ \label{a1:l6}
\end{algorithm}

\subsubsection{\textbf{KV Retrieval}} 
The retrieval incurred at \textbf{\emph{decoding}} has two stages (see \autoref{alg:DIDX-decoding}): \emph{Recall} and \emph{Rerank}. 
In the Recall stage (\autoref{fig:detail}(b)), we compute the cosine similarity between the query $Q$ and all query centroids in $Q^{(c)}$ (\autoref{a2:l1}--\autoref{a2:l2}), selecting the top-$C'$ centroids per KV head (\autoref{a2:l3}). 
The corresponding Key indices are deduplicated and gathered to produce recall Keys $K^\text{Recall}$ (\autoref{a2:l4}).
In the Rerank stage (\autoref{fig:detail}(c)), attention scores are computed for the recalled Keys, and the top-$\rho'$ Key IDs with the highest scores are selected as $I^\text{Rerank}$ (\autoref{a2:l5}--\autoref{a2:l6}). This avoids full attention computation over all recalled Keys. 

Attention computation involves three parts: QK, Softmax and WV. 
For Recall and Rerank, with Key lengths $L^\text{Recall} = \rho \cdot C' \cdot \alpha$~\footnote{Here, $\alpha$ refers the deduplication coefficient and typically equals 0.5 or less in practice. For detailed results, see Appendix \ref{sec:appendix_alpha} for $\alpha$ ranges under different $\rho$, $C$ and $C'$.} and $L^\text{Rerank}=\rho'$, full attention on recalled KV cache has a complexity of $\mathcal{O}(L^\text{Recall}\cdot1\cdot d)+\mathcal{O}(L^\text{Recall}\cdot1)+\mathcal{O}(L^\text{Recall}\cdot1\cdot d) = \mathcal{O}(2\cdot L^\text{Recall}\cdot d)$. 
In contrast, \model{} reduces this to $\mathcal{O}(L^\text{Recall} \cdot d)$ for rerank QK computation and $\mathcal{O}(2 \cdot L^\text{Rerank} \cdot d)$ for final attention, achieving an acceleration factor of $\frac{L^\text{Recall} + 2 \cdot L^\text{Rerank}}{2 \cdot L^\text{Recall}}$.
% the total time complexity includes the rerank QK computation $\mathcal{O}(1 \cdot L^\text{Recall} \cdot d)$ and subsequent attention computation $\mathcal{O}(2\cdot L^\text{Rerank} \cdot d)$, achieving an acceleration factor of $\frac{L^\text{Recall} + 2 * L^\text{Rerank}}{2 * L^\text{Recall}}$.
For example, for $L^\text{Recall}=10\text{K}$ and $L^\text{Rerank}=1\text{K}$ typically, rerank reduces end-to-end token generation time by approximately 40\%.

\begin{algorithm}[htb]
\fontsize{8.5}{11}\selectfont
\caption{\small{KV Retrieval at Decoding}}
\label{alg:DIDX-decoding}
\KwIn{$\QCIVF{} \in \mathbb{R}^{b \times g \times C \times \rho}$, $Q \in \mathbb{R}^{b \times h \times 1 \times d}$, number of retrieved centroids $C'$, number of retrieved Keys per query (\emph{sparsity budget}) $\rho'$.}
\BigComment{\bigcommentstyle{(b) $Q2Q^{(c)}$ Recall Procedure}}
\Comment{\commentstyle{Compute cosine similarity to each query centroid}}
$\mathit{Cos}_q\in \mathbb{R}^{b \times (h/g \times g) \times C \times 1} \leftarrow \operatorname{cos}(Q^{(c)}, Q)$  \label{a2:l1} \\
$\mathit{Cos}_{kv} \in \mathbb{R}^{b \times g \times C} \leftarrow \operatorname{max_\text{kv-group}}(\mathit{Cos}_q)$ \label{a2:l2}\\
\Comment{\commentstyle{Select top-$C'$ query centroid IDs for each KV head}}
$I^Q \in \mathbb{R}^{b \times g \times C'} \leftarrow \mathit{Cos}_{kv}.\operatorname{top}(C')$ \label{a2:l3}\\
\Comment{\commentstyle{Gather and dedup Keys of chosen query centroid}}
% $I^\text{Recall} \leftarrow \operatorname{DedupIdx}(\operatorname{Gather}(\QCIVF{}, I^Q))$ \label{a2:l4}\\
% $K^\text{Recall} \leftarrow \operatorname{Gather}(K^\text{CPU}, I^\text{Recall})$ \label{a2:l5} \\ 
$K^\text{Recall} \leftarrow \operatorname{Gather}(K^\text{CPU}, \operatorname{DedupIdx}(\QCIVF{}, I^Q))$\label{a2:l4}\\
\BigComment{\bigcommentstyle{(c) $Q2K$ Rerank Procedure}}
\Comment{\commentstyle{Use attention score to rerank the Keys to get top-$\rho'$ Key IDs for each KV head}}
$\mathcal{A}'_{kv} \leftarrow \operatorname{max}_\text{kv-group}(\operatorname{Softmax}(Q \cdot K^\text{Recall}))$ \label{a2:l5} \\
$I^\text{Rerank}\in \mathbb{R}^{b \times g \times \rho'} \leftarrow \mathcal{A}'_{kv}.\operatorname{top}(\rho')$ \label{a2:l6}\\
$K^\text{sparse} \leftarrow \operatorname{Gather}(K^\text{CPU}, I^\text{Rerank})$ \label{a2:l7}\\
$V^\text{sparse} \leftarrow \operatorname{Gather}(V^\text{CPU}, I^\text{Rerank})$ \label{a2:l8}\\
\BigComment{\bigcommentstyle{Attention Computation}}
$O^\text{CPU} \leftarrow \operatorname{Attn}(Q, K^\text{sparse}, V^\text{sparse})$ \label{a2:l9}\\ 
$O^\text{GPU} \leftarrow \operatorname{Attn}(Q, K^\text{GPU}, V^\text{GPU})$ \label{a2:l10}\\
$O \leftarrow \operatorname{Merge}(O^\text{CPU},O^\text{GPU})$ \label{a2:l11} \\
\BigComment{\bigcommentstyle{(d) Dynamic Centroid Update}}
\Comment{\commentstyle{Asynchronous computed after (c) and could be fully overlapped with attention computation}}
% $I^{IVF} \leftarrow \operatorname{top-k}(S, \rho)$ \label{a2:l13} \\
$\QCIVF{}.\operatorname{FIFO}(\mathcal{A}'_{kv}.\operatorname{top}(\rho))$ \label{a2:l12}
\end{algorithm}

\subsubsection{\textbf{Dynamic Centroid Update}} 
In tasks involving multi-turn conversations or generating long outputs, the similarity between the current query and the queries in \QCIVF{} gradually decreases as the positional distance grows, leading to less accurate retrieval of $K^\text{Recall}$ (\autoref{a2:l4}).
To address this, \model{} introduces an asynchronous procedure (\autoref{a2:l12}), starting after $Q2K$ Rerank step, to dynamically update \QCIVF{}.
This is achieved by adding the most relevant query-Key mappings identified through Recall and Rerank in each iteration (\autoref{a2:l5}). 
Notably, the time required for centroid update is fully overlapped with the attention computation, ensuring no additional latency. \QCIVF{} is managed using a first-in-first-out (FIFO) queue, which removes the oldest query-Key mappings to accommodate new ones, ensuring relevance to the current query.

\subsection{System-wide Optimization} \label{sec:sys}

\begin{figure}
    \centering
    \includegraphics[width=\linewidth]{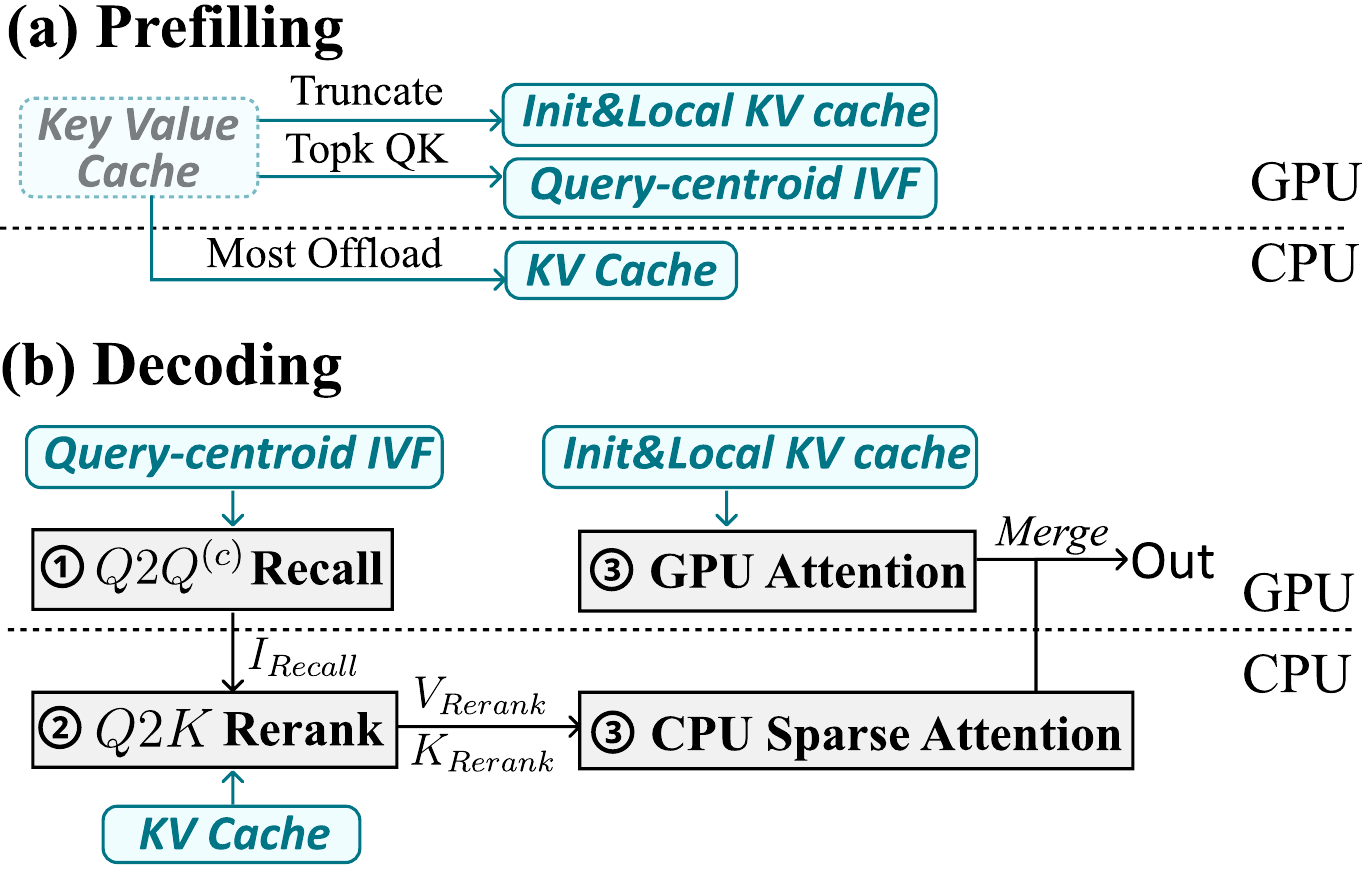}
    \caption{\model{} optimizes long-context decoding throughput by offloading most of the KV cache to the CPU during prefilling and enabling GPU-CPU co-execution for efficient attention computation.}
    \label{fig:overview}
\end{figure}

The KV cache size remains a critical bottleneck in long-context LLM decoding, especially when GPU VRAM is limited. However, advances in CPU FLOPs, bandwidth, and sparse attention mechanisms have made CPU-GPU co-execution for attention computation both feasible and efficient. 

As illustrated in \autoref{fig:overview}, \model{} introduces a novel system-wide optimization strategy.
During prefilling, the KV cache is split into two disjoint sets: a static GPU-resident cache containing initial and local tokens, and a larger portion offloaded to the CPU.
The GPU-resident cache adopts a static storage pattern \citep{xiao2023streamingllm}, while leaving room for integration with more complex patterns \citep{li2024snapkv} and \citep{minference}.

During decoding, \model{} performs $Q2Q^{(c)}$ recall (\autoref{fig:detail}(b)) on GPU and $Q2K$ Rerank (\autoref{fig:detail}(c)) on CPU to identify the relevant tokens residing on CPU.
Static and sparse portions of attention are computed on GPU and CPU, respectively, and results are merged to produce the final output.
To further accelerate decoding, \model{} includes a carefully optimized custom CUDA kernel ($\sim$200 LoC C++) designed specifically for efficient index deduplication and token retrieval (\autoref{a2:l4} in \autoref{alg:DIDX-decoding}) and utilizes CUDA multistreaming to overlap GPU and CPU attention computations, fully utilizing hardware resources to maximize throughput.

\section{Experiments} \label{sec:exp}

In this section, we showcase the effectiveness and efficiency of \model{}, specifically,
\begin{itemize}[leftmargin=*]
% [leftmargin=*,itemsep=-0.5em,topsep=0em]
\item \model{} maintains nearly full KV accuracy with less than 1\% degradation across moderate to long context tasks (Section \ref{sec:acc}).
\item \model{} provides exceptional system performance, supporting up to 20$\times$ larger batch sizes and achieving up to 4$\times$ higher inference throughput (Section \ref{sec:efficient}).
\item Extensive ablation studies validate the effectiveness of \model{}'s components and the impact of critical design parameters (Section \ref{sec:abla}).
\end{itemize}

\begin{table}[tbh]
\centering
% \footnotesize
\caption{Accuracy comparison of different methods across varying context lengths on \textsf{Ruler}.}
\label{tb:ruler}
\begin{tblr}{
  cell{2}{1} = {r=8}{},
  cell{10}{1} = {r=8}{},
  vline{2-3,8} = {-}{},
  hline{1,2,3,8,10,11,16,18} = {-}{},
  colsep=4pt, rowsep=0.5pt,
  row{2,10} = {bg=mygray},
  row{8,9,16,17} = {bg=lightblue},
  cell{2,10}{1} = {bg=white},
  column{2-8} = {c}, 
}
& \textbf{Context Len.} & \textbf{8K} & \textbf{16K} & \textbf{32K} & \textbf{64K} & \textbf{128K} & \textbf{AVG.} \\
\rotatebox{90}{\texttt{Llama-3-8B-262K}} & \textsc{FullKV}                 & 90.97       & 90.10        & 86.17        & 83.06        & 79.65         & 85.99        \\
                    & \textsc{Snap}                       &  21.17           &   16.11           &   5.90           &    2.58          &     1.00          &     -76.64        \\
                    & \textsc{Quest}                        &  88.81           &   86.01           & 81.16             &    75.63          &      70.48         &    -5.58           \\
                    & \textsc{MagicPIG}                &  85.04           &   85.00          &    80.34          &   75.27           &   67.51            &   -7.36         \\
                    & \textsc{ShadowKV}                     & 90.06            &  88.81            & 85.74              &   80.40           &   74.39            &    -2.11           \\
                    & \textsc{Flat}                       &  90.10           &   89.89           & 86.00             &   83.37           &   77.34            &    -0.65           \\
                    & {\model{}$_{512}$}                &  89.90           &   89.65           &    86.42          &   82.71          &   74.82            &   -1.29           \\
                    & {\model{}$_{1024}$}                &  90.14           &   89.93           &    86.20          &   83.38           &   76.60            &   \textbf{-0.74}           \\
\rotatebox{90}{\texttt{Yi-9B-200K}}      & \textsc{FullKV}                      &   89.60          &     81.94         &    71.5          &     66.09         &   60.91            &   74.01           \\
                    & \textsc{Snap}                       &     7.77        &  9.03            &    7.10          &   7.05           &   2.83            &     -67.25          \\
                    & \textsc{Quest}                        &   84.85          &   75.77           &    63.23          &     58.17         &    52.22           &    -7.16           \\
                    & \textsc{MagicPIG}                     & 86.53            & 78.66             &   67.56            &   62.89          &   58.54            &   -3.17      \\
                    & \textsc{ShadowKV}                     &   86.96          &  79.25            &    68.19          &      65.12        &    57.69           &     -2.57          \\
                    & \textsc{Flat}                       &   90.30          &    81.72          &   71.67           &    65.02          &    58.82           &      -0.51         \\
                    & {\model{}$_{512}$}                &  90.16           &   81.45           &    72.57          &   64.09           &   57.03            &   -0.92           \\
                    & {\model{}$_{1024}$}                & 90.16            &   81.45           &    72.50         &  64.54            &    58.55           &    \textbf{-0.56}     \\  
\end{tblr}
\end{table}

% \begin{figure}
%     \centering
%     \includegraphics[width=0.8\linewidth]{figure/3niah.png}
%     \caption{Performance of \model{} in Needle-in-a-haystack.}
%     \label{fig:niah}
% \end{figure}

\subsection{Accuracy Evaluation}\label{sec:acc}

\subsubsection{\textbf{Setup}} 
We employ two widely adopted long-context LLMs: \texttt{Llama-3-8B-262K} \citep{gradientlongcontextllama3} and \texttt{Yi-9B-200K} \citep{ai2024yi} to ensure a fair and direct comparison with previous works. 
\model{} is configured with $C = \min(2048, \frac{dataset\ length}{16})$, $C' = 4$, $\rho=2.5*\rho'$ and $\rho' = \{ 1024, 512\}$ to balance accuracy and efficiency.
Our evaluation spans three challenging long-context benchmarks: 
(1) \textsf{RULER} \citep{hsieh2024ruler}, a benchmark allowing users to customize the dataset length to generate 13 tasks across four categories (retrieval, multi-hop tracing, aggregation, and QA). We evaluate performance on context lengths ranging from 8K to 128K.
(2) \textsf{LongBench} \citep{bai2024longbench}, a benchmark with 21 datasets in both English and Chinese. We focus on 13 English datasets across five task categories: single-doc QA (S-Doc), multi-doc QA (M-Doc), summarization (Summ), few-shot learning (Few-Shot), and code completion (Code). 
(3) \textsf{Needle-in-a-haystack} \citep{needle}, a benchmark designed to test models' ability to retrieve critical information (the `needle') from lengthy documents (the `haystack').

% and Needle In A Haystack, covering QA, multi-hop, reasoning, summarization, code completion.
\subsubsection{\textbf{Baseline}} 
We include five training-free baselines (No.~1--5) and one exact KNN retrieval method (No.~6): 
(1) \textsc{FullKV} based attentions without cache manipulation; 
(2) \textsc{SnapKV} \citep{li2024snapkv}, an eviction-based method that evaluates attention scores to discard KV cache entries;
(3) \textsc{Quest} \citep{tang2024quest}, a block-level sparse attention method that estimates block importance using query vectors while tracking minimal and maximal Key-Value entries within each block;
(4) \textsc{ShadowKV} \citep{sun2024shadowkv}: a block-level sparse attention method that uses mean pooling of Key cache blocks as landmarks to compute block scores;
(5) \textsc{MagicPIG} \citep{magicpig}, a token-level sparse attention method using \emph{Locality-Sensitive Hashing} (LSH) to efficiently locate crucial tokens. The LSH in \textsc{MagicPIG} is configured with 9 random hash functions and 120 hash tables\footnote{Since \textsc{MagicPIG} adopts an LSH-based sampling method, it enforces a fixed sparsity ratio of 4\% relative to the context length, rather than a fixed sparsity budget. To ensure a fairer comparison, we also evaluated CTKVR under the same 4\% sparsity ratio on the Ruler benchmark with LLaMA-3-8B shown in Appendix \ref{sec:dyn_spar}.}; and
(6) \textsc{Flat}: an exact KNN method performing linear scans of all Key-Value vectors to identify Keys with the highest attention scores.
All methods are evaluated with a sparsity budget $\rho' = 1024$ (i.e., size of KV caches retrieved for sparse attention, as in \autoref{a2:l6}--\autoref{a2:l8} in \autoref{alg:DIDX-decoding}), whereas \model{} is tested with both $\rho'$ sizes of $512$ and $1024$.\footnote{We provide additional experiments comparing \model{} to RetrievalAttention and SqueezeAttention separately in Appendix \ref{sec:retrievalattn} and Appendix \ref{sec:squeezeattn}.}
% which is the length of KV cache involved in attention computation, while \model{} is tested with both sparsity budgets of 512 and 1024 for comparison.

\begin{table}[tb]
\centering
% \footnotesize
\caption{Accuracy comparison of different methods across multiple task categories on \textsf{LongBench}.}
\label{tb:longbench}
\begin{tblr}{
  cell{2}{1} = {r=8}{},
  cell{10}{1} = {r=8}{},
  vline{2-3,8} = {-}{},
  hline{1,2,3,8,10,11,16,18} = {-}{},
  colsep=2pt, rowsep=0.5pt,
  row{2,10} = {bg=mygray},
  row{8,9,16,17} = {bg=lightblue},
  cell{2,10}{1} = {bg=white},
  column{2-8} = {c}, 
}
& \textbf{Task} & \textbf{S-Doc} & \textbf{M-Doc} & \textbf{Summ} & \textbf{Few-Shot} & \textbf{Code} & \textbf{AVG.} \\
\rotatebox{90}{\texttt{Llama-3-8B-262K}} & \textsc{FullKV}                 & 29.13       & 21.82        & 29.33        & 79.04        & 49.50         & 41.76        \\
                    & \textsc{Snap}                       &  5.42           &   6.25           &   12.35           &    40.12          &     27.06          & -23.52        \\
                    & \textsc{Quest}                        &  30.76           &   19.97           & 27.04             &    78.84          &      50.26         & -0.39        \\
                    & \textsc{MagicPIG}                     &   18.17          &    9.59          &   26.79            &  78.08           &        48.82       &  -5.47       \\
                    & \textsc{ShadowKV}                     & 30.73            &  17.88            & 30.71              &   78.95           &   49.15            & -0.27        \\
                    & \textsc{Flat}                       &  29.94           &   24.01           & 28.92             &   78.35           &   50.26            & +0.54        \\
                    & {\model{}$_{512}$}                     &  29.55           &   21.25           & 29.32             &   78.12           &   50.41            & \textbf{-0.03}        \\
                    & {\model{}$_{1024}$}                & 29.68       & 21.20        & 28.98        & 78.12        & 50.29         & -0.11        \\
\rotatebox{90}{\texttt{Yi-9B-200K}}      & \textsc{FullKV}                      & 28.04       & 38.54        & 25.86        & 83.50        & 69.81         & 49.15        \\
                    & \textsc{Snap}                       & 8.83           &   10.13           &   11.29           &    52.38          &     25.16          & -27.59        \\
                    & \textsc{Quest}                        & 28.19           &   37.39           & 25.97             &    83.34          &      68.46         & -0.48        \\
                    & \textsc{MagicPIG}                     & 26.67            &   38.21           &  24.74             &   83.66          &   69.14            &   -0.66      \\
                    & \textsc{ShadowKV}                     & 29.90            &  30.32            & 27.71              &   83.61           &   69.31            & -0.98        \\
                    & \textsc{Flat}                       &  29.02           &   37.99           & 26.29             &   83.61           &   69.12            & +0.05        \\
                    & {\model{}$_{512}$}                     &  28.72           &   38.16           & 26.15             &   83.61           &   69.72            & \textbf{+0.12}        \\
                    & {\model{}$_{1024}$}                & 28.77       & 38.08        & 26.24        & 83.61        & 69.35         & +0.06        \\
\end{tblr}
\end{table}

\subsubsection{\textbf{Analysis on \textsf{RULER}}} 
As shown in \autoref{tb:ruler}, \model{} achieves accuracy comparable to \textsc{FullKV} attention (within $\pm 1\%$) across varying context lengths on \textsf{Ruler} and different LLMs, owing to its ability to accurately retrieve crucial tokens. 
In contrast, eviction-based methods (\textsc{SnapKV}) and block-level approaches (\textsc{ShadowKV} and \textsc{Quest}) struggle to maintain stable performance under limited sparsity budgets as context lengths grow. 
Methods capable of token-level retrieval like \textsc{Flat} and \model{} demonstrate superior robustness, with only about 2.5\% performance degradation even at context lengths of 128K.

\subsubsection{\textbf{Analysis on \textsf{LongBench}}} 
As shown in \autoref{tb:longbench}, \model{} achieves nearly the same accuracy as \textsc{FullKV} ($\pm 1\%$) and even surpasses it in certain tasks. 
Other eviction-based and block-indexing methods, however, experience varying levels of performance degradation. Additionally, \model{} delivers more consistent performance across categories, with a maximum performance drop of only 1\%.

\subsubsection{\textbf{Analysis on \textsf{Needle-in-a-Haystack}}} 
On \textsf{Needle-in-a-Haystack} (\autoref{fig:niah}), \model{} effectively retrieves information from various positions across context windows ranging from 16K to 200K tokens. Detailed comparisons involving other baselines are provided in Appendix~\ref{sec:appendix_needle}.

\begin{figure}[!htbp]
    \centering
    \includegraphics[width=\linewidth]{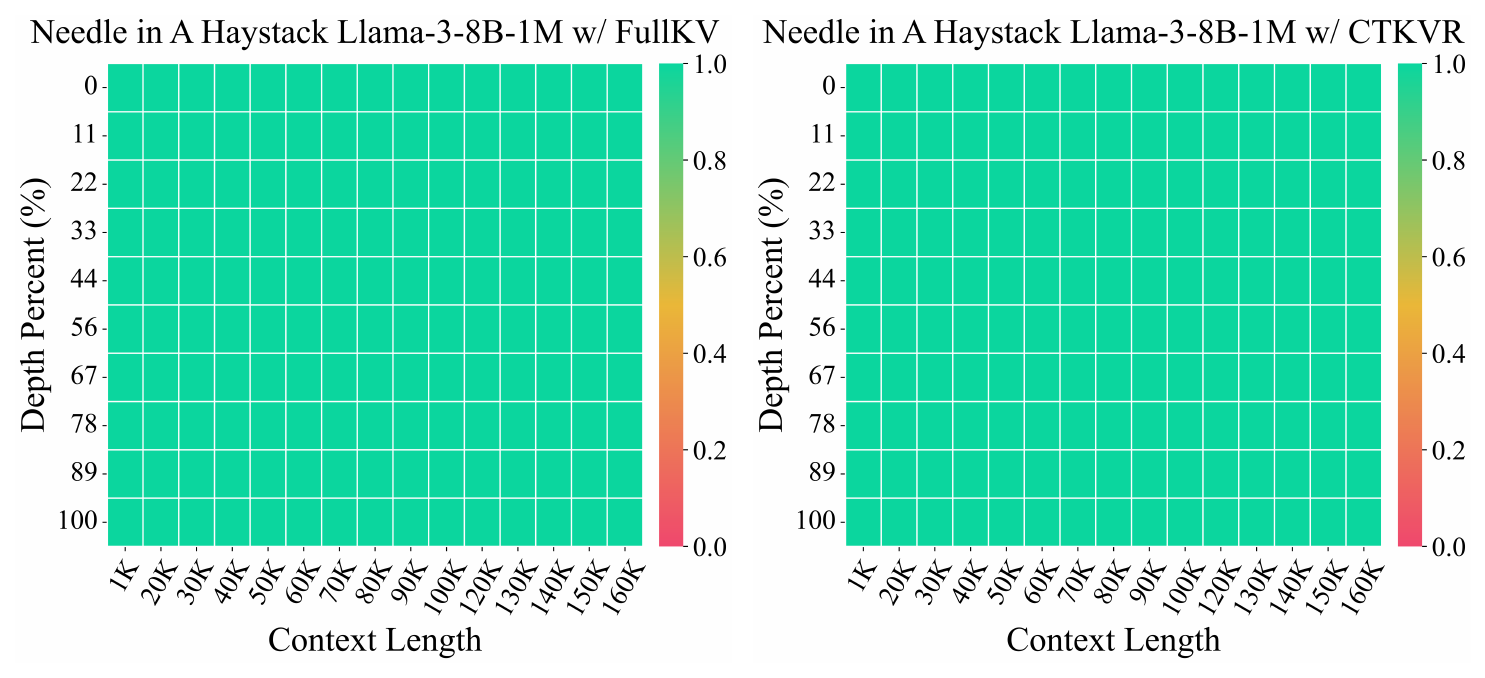}
    \caption{Performance comparison of \textsc{FullKV} and \model{} using heatmaps, following methodology originated from the \textsf{Needle-in-a-haystack} paper.}
    \label{fig:niah}
\end{figure}

\subsubsection{\textbf{Scaling up to Extremely Long-Context Inference}} 
We evaluate \model{} on the \textsf{Needle-in-a-haystack} dataset with extremely long contexts ranging from 200K to 1 million (1M) tokens, using \texttt{Llama-3-8B-1M} \citep{gradientlongcontextllama3}. As shown in \autoref{fig:longlong}, \model{} successfully retrieves all needles, demonstrating the robustness and effectiveness of our indexing methods in handling ultra-long contexts.

\begin{figure}[!htbp]
    \centering
    \includegraphics[width=0.6\linewidth]{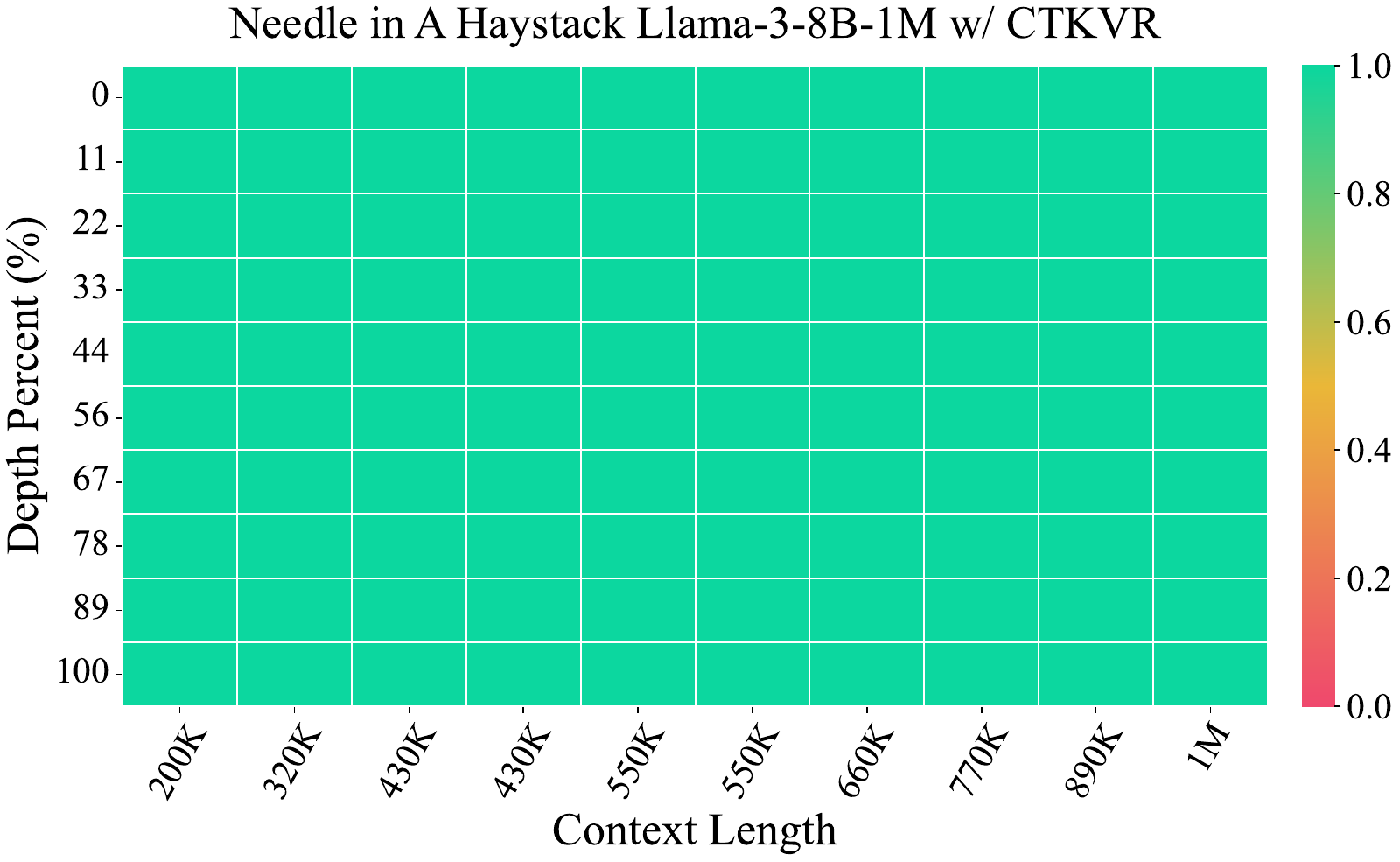}
    \caption{Performance of \model{} on \textsf{Needle-in-a-haystack} with context lengths ranging from 200K to 1M tokens, evaluated on \texttt{Llama-3-8B-1M}.}
    \label{fig:longlong}
\end{figure}

\subsubsection{\textbf{Integration with Efficient Prefilling Methods}}
We further evaluate the performance of \model{} when integrated with the state-of-the-art efficient prefilling method, \textsc{mInference} \citep{minference}. Following the experimental setup of \textsc{mInference}, we test both methods on the \textsf{RULER} dataset with context lengths ranging from 8K to 128K.

As shown in \autoref{tb:prefill}, the results demonstrate that \model{} is fully compatible with the prefilling acceleration techniques, exhibiting less than 1\% performance degradation across tested context lengths. Notably, \model{} even improves performance at certain lengths, such as 32K and 64K, further validating its adaptability and efficiency in long-context scenarios.

\begin{table}[!htbp]
\centering
\footnotesize
\begin{tblr}{
    colspec = {c c c c c c c},
    row{1} = {font=\bfseries}, 
    hline{1,2,4} = {-}{},
    vline{2,7} = {-}{},
    colsep=3.5pt, rowsep=1.5pt,
    row{2} = {bg=mygray},
    cell{2}{1} = {bg=white},
}
Prefilling+Decoding Method               & 8K    & 16K   & 32K   & 64K   & 128K  & AVG.    \\
\textsc{mInference}+\textsc{FullKV}      & 90.79 & 89.78 & 85.54 & 82.25 & 78.1  & 85.29 \\
\textsc{mInference}+\model{}$_{512}$ & 89.98 & 89.51 & \textbf{85.93} & \textbf{82.42} & 75.32 & -0.65 \\
\end{tblr}
\caption{Accuracy comparison of different decoding methods combined with efficient prefilling methods \textsc{mInference} on \texttt{Llama-3-8B-262K}.}
\label{tb:prefill}
\end{table}

% \paragraph{Multi-turn Conversation Capability}

\subsection{System Efficiency Evaluation}
\label{sec:efficient}

\subsubsection{\textbf{Setup}} 
We evaluate our system across six configurations, combining two LLM setups and three GPU settings. 
Two LLMs, \texttt{Yi} with 9B parameters and \texttt{Llama-3} with 8B parameters are tested, both on a 96K context.
GPUs with different memory capacities include A6000 (48GB), V100 (32GB), and A6000 (24GB)\footnote{The performance of the 24GB GPU is simulated by imposing a memory limit on the A6000 due to a hardware lack. Tests indicate that this closely approximates the performance of standard 24GB GPUs, such as the RTX 4090.}. 
\model{} uses an Intel Xeon w9-3495X CPU with 56 cores and 250GB DRAM.

\subsubsection{\textbf{Baseline}} 
Following prior works \citep{magicpig, sun2024shadowkv, tang2024quest}, we compare \model{} against \textsc{FullKV} that preserves all KV caches in GPU for full attentions. 
For further illustrate the system efficiency of \model{}, we include two high-throughput implementations:
(1) \textsc{Vllm} \citep{vllm}, a Full-KV alternative that accelerates inference by sharing KV caches across requests;
(2) \textsc{MagicPIG}, which utilizes token-level indexing to offload KV caches and hash tables to the CPU and performs co-execution for the final attentions. 
For \textsc{MagicPIG}, we use the same hyperparameter settings as in Section~\ref{sec:acc}.
Additionally, to evaluate the effect of Rerank Module in an end-to-end scenario, we test the performance of \model{} with and without Rerank Module.
% , another \textsc{FullKV} method and one token-level indexing methods both 
% We compare \model{} against two full attentions and one token-level indexing sparse attention methods:
% (1) \textsc{Vllm} \citep{vllm}, a method that accelerates inference by sharing the KV cache across requests;
% (2) \textsc{FullKV}, a baseline that stores the entire KV cache on the GPU during both the prefilling and decoding stages; and 
% (3) \textsc{MagicPIG}, a method that offloads the KV cache and hash tables to the CPU and performs co-execution for the final attention computation. For \textsc{MagicPIG}, we use the same hyperparameter settings as in Section \ref{sec:acc}.

\begin{figure}
    \centering
    \includegraphics[width=\linewidth]{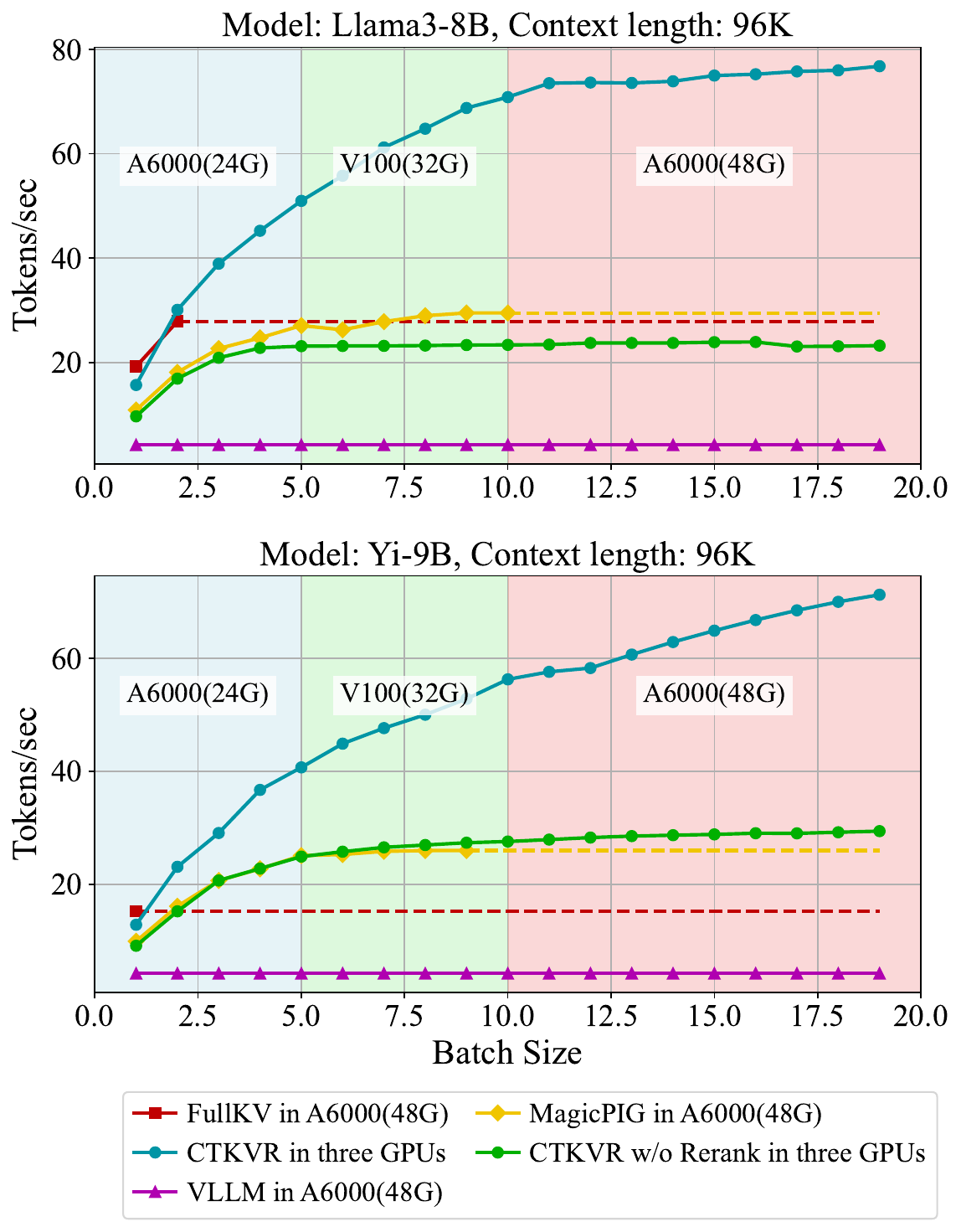}
    % \caption{Throughput comparison under different scenarios}
    \caption{On \texttt{Llama-3-8B} and \texttt{Yi-9B}, \model{} with $\rho' = 512$ support batch sizes up to 10$\times$ and 20$\times$, achieving throughput boost of 3$\times$ and 4$\times$, respectively.}
    \label{fig:throughput}
\end{figure}

\subsubsection{\textbf{End to End Efficiency Analysis}} 
As shown in \autoref{fig:throughput}, our analysis highlights the following observations:
(1) \model{} significantly improves decoding throughput, achieving up to 3$\times$ and 4$\times$ speedups on \texttt{Llama-3-8B} and \texttt{Yi-9B}, respectively. Also, \model{} enables decoding of 96K contexts even on a GPU with only 24GB VRAM.  
(2) By offloading the KV cache to CPU DRAM, \model{} supports substantially larger batch sizes, 10$\times$ for \texttt{Llama-3-8B} and 20$\times$ for \texttt{Yi-9B}, compared to \textsc{FullKV}. 
(3) Compared to \textsc{MagicPIG}, which also employs CPU-GPU co-execution, \model{} avoids storing resource-intensive hash tables on the CPU, thereby supporting up to 2$\times$ larger batch sizes. 
Additionally, due to its more efficient token-locating algorithm, \model{} achieves up to 2$\times$ higher throughput. 
(4) By employing the Rerank module, \model{} further boosts throughput by up to 2$\times$ on both LLMs.

\subsubsection{\textbf{Index Construction Efficiency Analysis}}
We compare the index construction time of \model{} with three approximate nearest neighbor (ANN) methods, \textsc{IVF}, \textsc{HNSW} and \textsc{KMeans}, implemented in the Faiss library.
The evaluation is conducted across varying lengths of Key vectors, ranging from 4K to 128K.

As shown in \autoref{fig:index_construct_speed}, \model{} separately achieves up to a $50\times$, $1000\times$ and $10000\times$ reduction in index construction time compared to \textsc{HNSW}, \textsc{IVF} and \textsc{KMeans} methods.
Meanwhile, \model{} demonstrates significantly greater stability, with smaller fluctuations in construction time as the context length increases, further showcasing its scalability and efficiency for long-context scenarios.

\begin{figure}[!htbp]
    \centering
    \includegraphics[width=\linewidth]{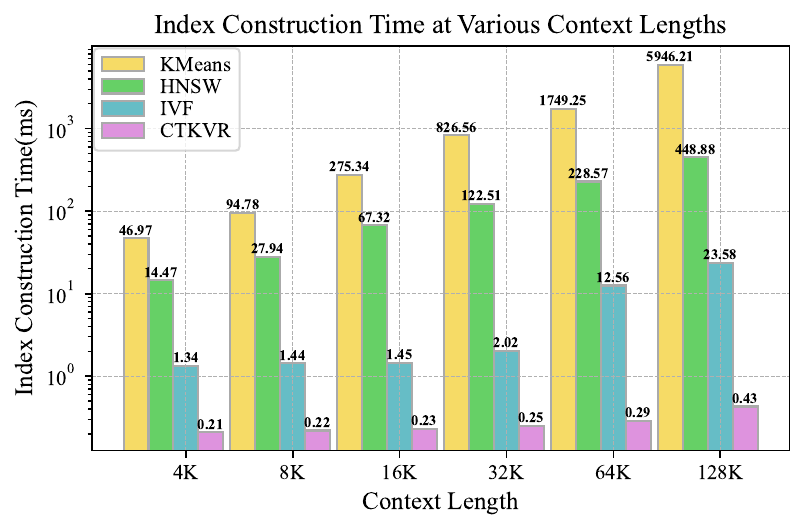}
    \caption{Index construction time of \model{} compared to ANN methods \textsc{HNSW} and \textsc{IVF} across varying Key vector lengths.}
    \label{fig:index_construct_speed}
\end{figure}

\subsection{Parameter and Component Study}
\label{sec:abla}

\subsubsection{\textbf{Parameter Study}}\footnote{Appendix~\ref{sec:tradeoff} provides additional results on the trade-off between computation and accuracy under varying parameters.}
We analyze the impact of four key parameters on \model{} using \textsf{RULER}:
% conduct three parameter studies of \model{} on the \textsf{Ruler} dataset to measure the accuracies across various tasks:  
% (1) sparse KV cache budget variations ($\rho'$);
% (2) number of maintained query centroids ($C$); and 
% (3) number of retrieved query centroids during decoding ($R_C$). 
\begin{itemize}[leftmargin=*]
% [leftmargin=*,itemsep=-0.5em,topsep=0em]
\item \textbf{\emph{Sparsity Budget} $\rho'$}. 
As shown in \autoref{fig:sparse_budget}, we evaluate \model{} under  different $\rho'$ values. 
\model{} consistently surpasses \textsc{ShadowKV} across all tasks, exhibiting smaller accuracy fluctuations across budget variations.
Notably, \model{} achieves near-\textsc{FullKV} accuracy with only budget $\rho' = 0.39\%$, even slightly improving measures on the task \emph{Question Answer 2}.

\item \textbf{\emph{Number of Maintained Centroids} $C$}. 
We analyze the impact of varying $C$ on \model{}'s accuracy.
As shown in \autoref{fig:centroid}(a), accuracy improves as $C$ increases, stabilizing near full KNN performance at around 320 centroids for most datasets, while for the dataset \textsf{niah\_single\_3}, performance continues to improve as $C$ grows. Meanwhile, throughput drops slightly with larger maintained $C$ for increasing time for $Q2Q^{(c)}$ Recall time only accounts for a minor component of the end-to-end time.

\item \textbf{\emph{Number of Retrieved Centroids} $C'$}. 
% We examine impact of $C'$ on \model{}'s accuracy. 
As shown in \autoref{fig:centroid}(b), accuracy improves with increasing $C'$, stabilizing near full KNN result at around 5 centroids, showing \model{}'s robustness even with a limited number of retrieved centroids.  However, throughput drops more sharply with larger retrieved centroids $C'$ for increasing time for $Q2K$ Rerank time constitutes the predominant portion of end-to-end latency.

\item \textbf{\emph{Position of $C$ maintained centroids}}. 
To illustrate the necessity of selecting the last $C$ centroids as indicated by $Obs. 2$, we added experiments comparing four query centroid selection strategies.
(1) \textsc{Final} (\model{}) uses last query vector in the sequence;
(2) \textsc{Random} selects query from random position;
(3) \textsc{Init} selects from the beginning;
(4) \textsc{Equi} selects query with evenly spaced positions.
We evaluate \model{} with each strategy on the \textsf{Ruler} benchmark using a 64k context and the Llama-3-8B model. As shown \autoref{tb:centroid_pos}, our \textsc{Final} strategy significantly outperforms other strategies.

\begin{table*}[!htbp]
\centering
% \footnotesize
\begin{tblr}{
  colsep=2pt, rowsep=0.8pt,
      cell{1}{1} = {r=2}{},
    cell{1}{15} = {r=2}{},
    cell{1}{2} = {c=3}{},
    cell{1}{5} = {c=3}{},
    cell{1}{8} = {c=3}{},
    cell{1}{11} = {c=2}{},
    cell{1}{13} = {c=2}{},
    row{1} = {font=\bfseries},
  column{2-15} = {c},
  vline{2,5,8,11,13,15} = {-}{},
  hline{1,3,7} = {-}{},
  % row{3} = {bg=mygray},
  row{6} = {bg=lightblue},
}
\textbf{Method} & \textbf{S-Doc} & & & \textbf{M-Doc} & & & \textbf{Summ} & & & \textbf{Few-Shot} & & \textbf{Code} &  & AVG. \\
 & \textit{NQA} & \textit{Multi\_en} & \textit{Qasper} & \textit{HQA} & \textit{Musique} & \textit{2wiki} & \textit{Multinews} & \textit{Qmsum} & \textit{Gov} & \textit{Trec} & \textit{TQA} & \textit{Lcc} & \textit{Repobench} &  \\
\textbf{\model{}$_{1024}^{Init}$}   & 91.67 & 100.00 & 39.58 & 95.83 & 96.88 & 94.27 & 95.83 & 77.08 & 48.96 & 80.56 & 80.83 & 1.04 & 80.21 & 75.60 \\
\textbf{\model{}$_{1024}^{Random}$} & 84.38 & 100.00 & 88.54 & 100.00 & 97.92 & 99.74 & 92.71 & 76.04 & 52.08 & 76.74 & 82.92 & 2.60 & 85.42 & 79.93 \\
\textbf{\model{}$_{1024}^{Equi}$}   & 96.88 & 100.00 & 95.83 & 100.00 & 97.92 & 99.48 & 94.53 & 78.13 & 51.04 & 74.31 & 86.25 & 4.90 & 86.46 & 81.98 \\
\textbf{\model{}$_{1024}^{Final}$}  & 100.00 & 100.00 & 98.96 & 100.00 & 97.92 & 99.48 & 96.09 & 77.08 & 53.13 & 80.90 & 90.42 & 1.04 & 88.54 & 83.35 \\
\end{tblr}
\caption{Accuracy comparison of four query centroid selection strategies across each tasks in \textsf{LongBench}.}
\label{tb:centroid_pos}
\end{table*}

\end{itemize}

\begin{figure}[ht]
    \centering
    \includegraphics[width=\linewidth]{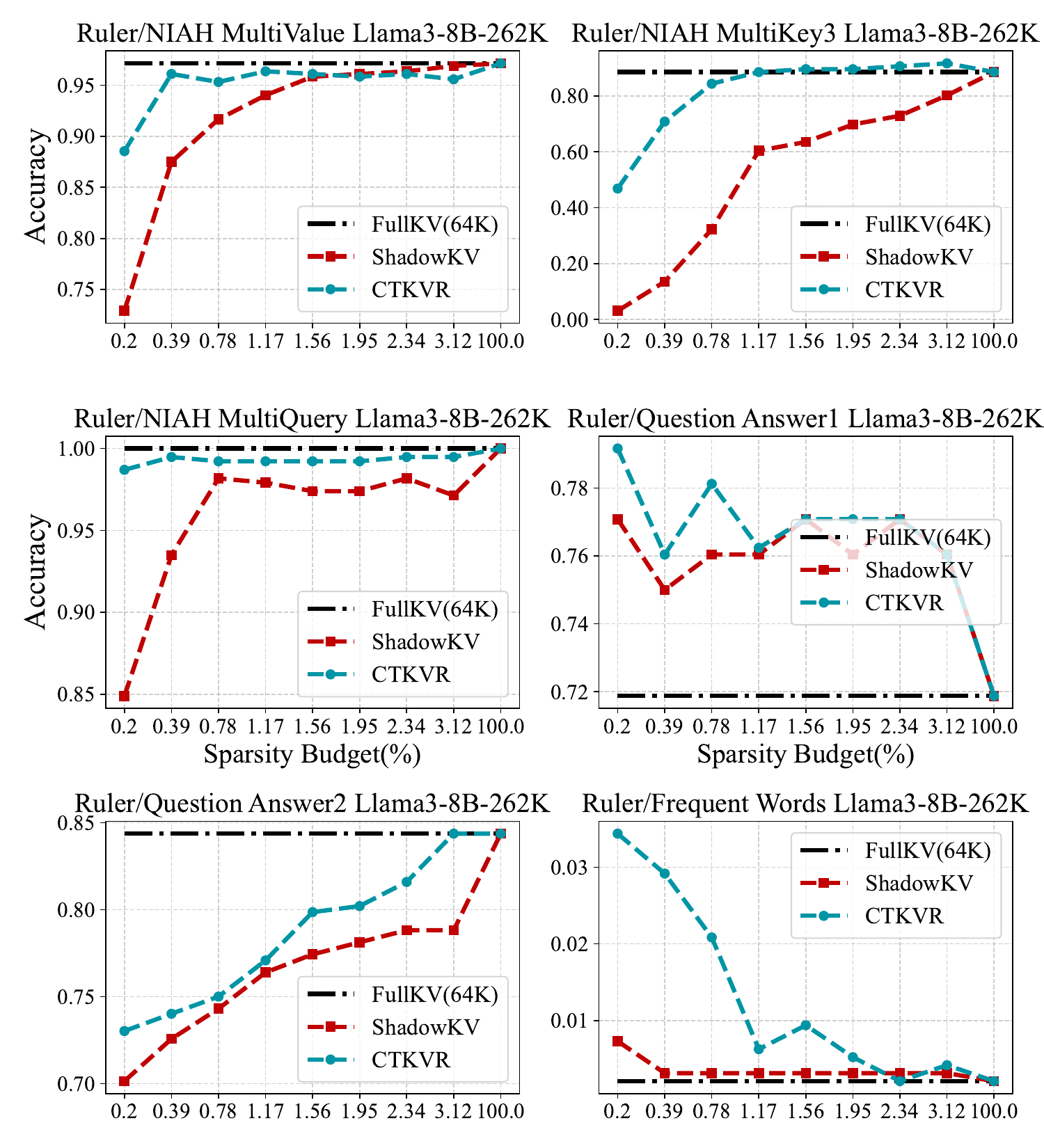}
    \caption{Accuracy vs sparsity budget ($\rho$) on \textsc{FullKV}, \textsc{ShadowKV}, and \model{}.}
    \label{fig:sparse_budget}
\end{figure}

% \paragraph{Sparse KV Cache Budget ($\rho'$)} 
% As shown in \autoref{fig:sparse_budget}, we examine \model{}'s performance with different budgets. 
% \model{} consistently surpasses \textsc{ShadowKV} across all tasks, demonstrating smaller accuracy fluctuations across different budgets.
% Specifically, \model{} achieves near \textsc{FullKV} performance with only 0.39\% sparsity budget utilization across all tasks, even slightly improving performance on the task \emph{Question Answer 2}.

% \paragraph{Number of Maintained Query Centroids ($C$)}  
% We analyze the impact of varying the number of maintained query centroids $C$ on \model{}'s accuracy.
% As shown in \autoref{fig:centroid}(a), accuracy improves with increasing $C$, stabilizing near full KNN performance at approximately 320 centroids for most datasets. 
% Meanwhile, for the dataset \textsf{niah\_single\_3}, performance continues to improve as $C$ increases.

\begin{figure}[ht]
    \centering
    \includegraphics[width=\linewidth]{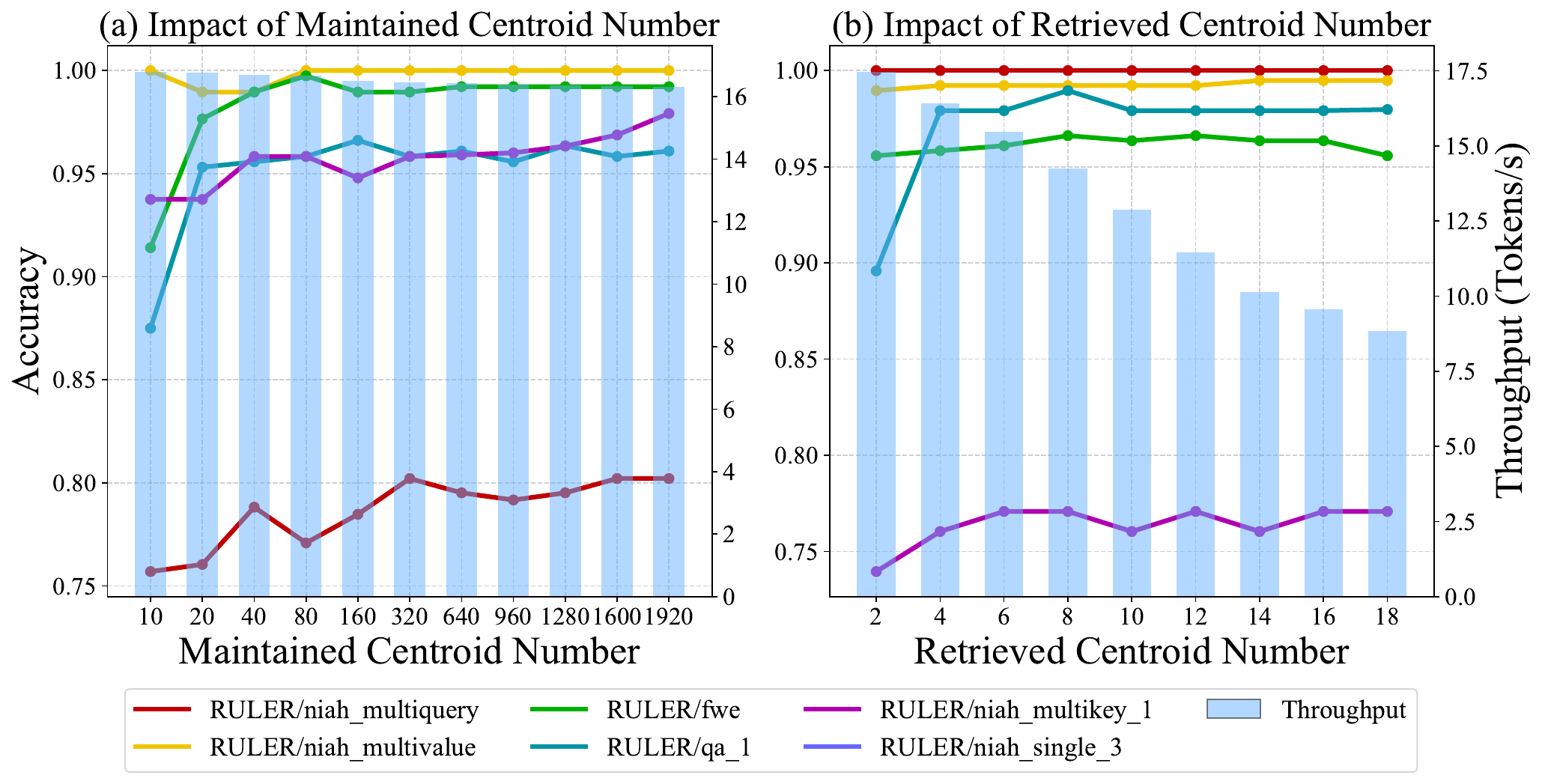}
    \caption{Impact of maintained centroid number $C$ and retrieved centroid number $C'$ on accuracy.}
    \label{fig:centroid}
\end{figure}

% \paragraph{Number of Retrieved Query Centroids During Decoding ($R_C$)} 
% We evaluate \model{}'s performance with varying numbers of retrieved query centroid $R_C$. 
% As shown in \autoref{fig:centroid}(b), accuracy improves with increasing $R_C$, stabilizing near full KNN performance at approximately 5 centroids.

\subsubsection{\textbf{Component Study}}
% We also evaluate the CPU speedup achieved by the Rerank module under different $\frac{L^{Recall}}{L^{Rerank}}$ ratios.
We conduct ablation studies to verify the efficacy of two \model{} modules:
\begin{itemize}[leftmargin=*]
% [leftmargin=*,itemsep=-0.5em,topsep=0em]
\item \textbf{\emph{Acceleration from the Rerank Module}}. 
We analyze the CPU speedup achieved by the Rerank module under different batch sizes and varying ratios of $\frac{L^{\text{Recall}}}{L^{\text{Rerank}}}$.
As shown in \autoref{fig:rerank}, the speedup grows with a higher Rerank ratio, reaching up to 2$\times$ acceleration, consistent across all batch sizes.
This demonstrates the effectiveness of the Rerank module in optimizing computational efficiency.

\item \textbf{\emph{Enhancing Multi-turn Conversations with Dynamic Centroid Update (DCU)}}.
We evaluate the impact of DCU in multi-turn conversation scenarios.
To simulate this, we challenged \model{} with a multi-turn needle retrieval task (Multi-turn NIAH)\footnote{We create the dataset by modifying the dataset generation script from \textsf{RULER}, generating multiple question-answer pairs for each context for multi-turn conversations.}. As shown in \autoref{tb:multi-turn}, incorporating DCA leads to noticeable performance improvements across different rounds of conversation (ranging from 2 to 8 turns). This highlights the module's ability to dynamically adapt to evolving conversational contexts.
\end{itemize}

\begin{figure}[!htbp]
    \centering
    \includegraphics[width=0.9\linewidth]{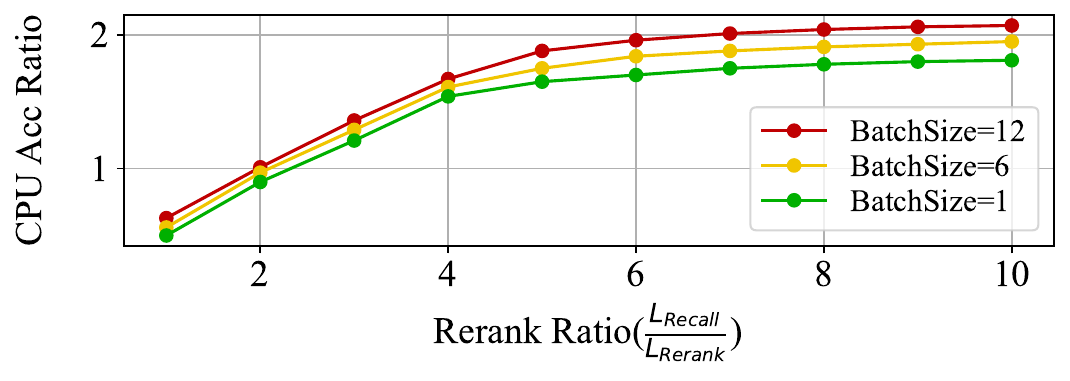}
    \caption{Rerank ratio $\frac{L^\text{Recall}}{L^\text{Rerank}}$ vs CPU speedups.}
    \label{fig:rerank}
\end{figure}

% \begin{table}[!htbp]
% \footnotesize
% \caption{Ablation of DCU in multi-turn conversation.}
% \label{tb:multi-turn}
% \centering
% \begin{tblr}{
%   vline{2} = {-}{},
%   hline{1-2,4} = {-}{},
%   colsep=4pt, rowsep=1.5pt,
% }
% \textbf{Conversation Round} & \textbf{1}   & \textbf{2}     & \textbf{4}     & \textbf{6}     & \textbf{8}     \\
% \textbf{\model{}  }                      & \textbf{100} & \textbf{95.78} & \textbf{94.38} & \textbf{94.18} & \textbf{94.02} \\
% \textbf{\model{} w/o DCA    }            & 100          & 90.02          & 88.45          & 86.27          & 85.67          
% \end{tblr}
% \end{table}

\begin{table}[!htbp]
\centering
% \footnotesize
\caption{Ablation of DCU in multi-turn conversation.}
\label{tb:multi-turn}
\begin{tblr}{
  colspec={lccccc},
  hline{1-2,4} = {-}{},
  vline{2} = {-}{},
  colsep=4pt, rowsep=0.75pt,
}
Round  & $1$ & $2$ & $4$ & $6$ & $8$ \\
{\model{} w/o DCU}    & 100        & 90.02      & 88.45      & 86.27      & 85.67      \\
{\model{}}            & \textbf{100} & \textbf{95.78}      & \textbf{94.38} & \textbf{94.18} & \textbf{94.02} \\
\end{tblr}
\end{table}
\section{Related Work}

\noindent{\textbf{KV Cache Eviction}} 
These methods reduce memory footprints by retaining critical portions of the KV cache while discarding less relevant entries.
\citet{xiao2023streamingllm} evict tokens farthest from the current focus based on positional locality.
\textsc{H\textsubscript{2}O} \citep{zhang2023ho} and \textsc{SnapKV} \citep{li2024snapkv} prioritize eviction using cumulative attention scores to remove irrelevant tokens.
\citet{huang2024locretenhancingevictionlongcontext} employ an MLP to predict tokens for eviction, while \citet{cai2024pyramidkvdynamickvcache} and \citet{feng2024adakvoptimizingkvcache} allocate storage budgets dynamically across layers or heads based on informativeness.
Unlike token eviction methods, \model{} avoids significant information loss by leveraging sparse attention to selectively retrieve key-value pairs, ensuring both efficiency and high accuracy.

% For memory footprint reduction, KV Cache eviction method keep a constant length of KV Cache for storing the critical portion and discard the unimportant. StreamingLLM \citep{xiao2023streamingllm} retaining initial and local tokens, evicting farthest local token positionally each time. H\textsubscript{2}O \citep{zhang2023ho} and SnapKV \citep{li2024snapkv} both updating the KV Cache based on centain policy for cumulative attention scores, guranteeing eviction most irrelvant KV Cache. Locret \citep{huang2024locretenhancingevictionlongcontext} train MLP layer to predict which tokens need to be evicted during each prediction. PyramidKV \citep{cai2024pyramidkvdynamickvcache} and Ada-KV \citep{feng2024adakvoptimizingkvcache} respectively allocating different storage budgets across different layers and heads, ensuring layers or heads with rich information are allocated larger storage budgets. However, they suffer from accuracy degradation for information loss in token eviction.

\noindent{\textbf{Block-Level Sparse Attention}} 
These methods summarize the Key cache into compact representations (e.g., mean-pooling) and retrieve blocks based on their similarity to query vectors.
\textsc{Quest} \citep{tang2024quest} and \textsc{Inf-LLM} \citep{xiao2024infllm} focus on mean-pooling crucial keys, while \textsc{QuickLlama} \citep{li2024quickllamaqueryawareinferenceacceleration} utilizes local attention scores for retrieval.
% mean-pooling the most important Keys as summary vector, evaluating with certain policy on attention score with block. Base this, QuickLlama \citep{li2024quickllamaqueryawareinferenceacceleration} considering attention score to local window of KV Cache as partial policy. 
\citet{sun2024shadowkv} simplifies this further by mean-pooling all tokens in each block.
% Although effectively accelerate the retrieve procedure, it may introduce some irrelevant KV cache within the block, leading to accuracy degradation.
\model{} overcomes block-level limitations by using fine-grained token-level indexing, ensuring precise retrieval of relevant tokens without introducing noise from irrelevant blocks.

\noindent{\textbf{Token-Level Sparse Attention}} 
These methods use ANN search or clustering to retrieve top-$k$ similar KV pairs during decoding.
\citet{liu2024retrievalattentionacceleratinglongcontextllm} leverage \textsc{RoarGraph} \citep{roargraph} for ANN, while \citet{hooper2024squeezed} incorporates hierarchical clustering for gradual token retrieval.
\textsc{MagicPIG} \citep{magicpig} applies LSH to sample and retrieve significant tokens efficiently.
\model{} achieves superior scalability by adopting a lightweight query-centric indexing mechanism, avoiding the computational overhead in LSH or hierarchical clustering while maintaining high retrieval accuracy and throughput. Meanwhile, the lightweight index in \model{} avoids creating large HashMaps in CPU DRAM for LSH, enabling a larger batch size for long-context inference.

% Token-level strategies establishes ANN index or clusters Key caches during prefilling and retrieves the top-k similar Key cache by ANN or KNN search during decoding. RetrievalAttention \citep{liu2024retrievalattentionacceleratinglongcontextllm} use the state-of-the-art cross-modal ANNS index RoarGraph \citep{roargraph} for constructing index and searching the top-k KV Cache within vector-db each time. SqueezeAttention \citep{hooper2024squeezed} construct hierarchical clustering by k-means and gradually indexing for top-k token during decoding. MagicPIG \citep{magicpig} use LSH sampling to retrieve important tokens.
\section{Conclusion} \label{sec:conclusion}

We present \model{}, a novel framework for efficient and accurate sparse attention in long-context LLMs. By leveraging the high similarity between adjacent query vectors, \model{} introduces a two-stage retrieval mechanism: query-centroid indexing followed by token-level indexing. This ensures both fast retrieval and high-quality KV cache entries.
Extensive experiments show that \model{} achieves superior accuracy with less than 1\% degradation compared to full KV cache methods across various benchmarks while delivering up to 3$\times$ and 4$\times$ throughput gains on \texttt{Llama-3-8B} and \texttt{Yi-9B}, respectively, at a 96K context length on various standard GPUs.

% In this work, we first demonstrate the adjacent query vector in positions tend to exhibit high similarity, therefore share most top-$k$ KV cache entries. Based on this, we present \model{}, a method for constructing query-centroid IVF during prefilling and performing query-centroid indexing then token indexing, ensuring both the retrieval speed and quality of retrieved KV cache entries. Experiments shows \model{} achieves superior accuracy across various benchmarks with less than 1\% degradation. Meanwhile, \model{} provides a $3\times$ and $4\times$ throughtput improvement separately on \texttt{Llama-3-8B} and \texttt{Yi-9B} across various GPU hardwares with 96K context length.

\clearpage
\bibliographystyle{ACM-Reference-Format}
\bibliography{main}

@misc{openai2024gpt4technicalreport,
      title={GPT-4 Technical Report}, 
      author={OpenAI and Josh Achiam and Steven Adler and Sandhini Agarwal and Lama Ahmad},
      year={2024},
      eprint={2303.08774},
      archivePrefix={arXiv},
      primaryClass={cs.CL},
      url={https://arxiv.org/abs/2303.08774}, 
}

@misc{grattafiori2024llama3herdmodels,
      title={The Llama 3 Herd of Models}, 
      author={Aaron Grattafiori and Abhimanyu Dubey and Abhinav Jauhri and Abhinav Pandey and Abhishek Kadian and Ahmad Al-Dahle and Aiesha Letman },
      year={2024},
      eprint={2407.21783},
      archivePrefix={arXiv},
      primaryClass={cs.AI},
      url={https://arxiv.org/abs/2407.21783}, 
}

@misc{geminiteam2024geminifamilyhighlycapable,
      title={Gemini: A Family of Highly Capable Multimodal Models}, 
      author={Gemini Team and Rohan Anil and Sebastian Borgeaud and Jean-Baptiste Alayrac and Jiahui Yu and Radu Soricut},
      year={2024},
      eprint={2312.11805},
      archivePrefix={arXiv},
      primaryClass={cs.CL},
      url={https://arxiv.org/abs/2312.11805}, 
}

@misc{chatbot,
      title={A Complete Survey on LLM-based AI Chatbots}, 
      author={Sumit Kumar Dam and Choong Seon Hong and Yu Qiao and Chaoning Zhang},
      year={2024},
      eprint={2406.16937},
      archivePrefix={arXiv},
      primaryClass={cs.CL},
      url={https://arxiv.org/abs/2406.16937}, 
}

@inproceedings{bai2024longbench,
    title = "{L}ong{B}ench: A Bilingual, Multitask Benchmark for Long Context Understanding",
    author = "Bai, Yushi and Lv, Xin  and Zhang, Jiajie  and Lyu, Hongchang  and
      Tang, Jiankai  and Huang, Zhidian  and Du, Zhengxiao  and Liu, Xiao  and Zeng, Aohan  and Hou, Lei  and Dong, Yuxiao  and Tang, Jie  and Li, Juanzi",
    booktitle = "Proceedings of the 62nd Annual Meeting of the Association for Computational Linguistics (Volume 1: Long Papers)",
    month = aug,
    year = "2024",
    address = "Bangkok, Thailand",
    publisher = "Association for Computational Linguistics",
    url = "https://aclanthology.org/2024.acl-long.172",
    doi = "10.18653/v1/2024.acl-long.172",
    pages = "3119--3137",
}

@inproceedings{h2o,
author = {Zhang, Zhenyu and Sheng, Ying and Zhou, Tianyi and Chen, Tianlong and Zheng, Lianmin and Cai, Ruisi and Song, Zhao and Tian, Yuandong and R\'{e}, Christopher and Barrett, Clark and Wang, Zhangyang and Chen, Beidi},
title = {H2O: heavy-hitter oracle for efficient generative inference of large language models},
year = {2024},
booktitle = {Proceedings of the 37th International Conference on Neural Information Processing Systems},
articleno = {1506},
numpages = {50},
location = {New Orleans, LA, USA},
series = {NIPS '23}
}

@article{bai2024longwriter,
  title={LongWriter: Unleashing 10,000+ Word Generation from Long Context LLMs}, 
  author={Yushi Bai and Jiajie Zhang and Xin Lv and Linzhi Zheng and Siqi Zhu and Lei Hou and Yuxiao Dong and Jie Tang and Juanzi Li},
  journal={arXiv preprint arXiv:2408.07055},
  year={2024}
}

@misc{munkhdalai2024leavecontextbehindefficient,
      title={Leave No Context Behind: Efficient Infinite Context Transformers with Infini-attention}, 
      author={Tsendsuren Munkhdalai and Manaal Faruqui and Siddharth Gopal},
      year={2024},
      eprint={2404.07143},
      archivePrefix={arXiv},
      primaryClass={cs.CL},
      url={https://arxiv.org/abs/2404.07143}, 
}

@article{xiao2024infllm,
  author       = {Chaojun Xiao and Pengle Zhang and Xu Han and Guangxuan Xiao and Yankai Lin and Zhengyan Zhang and Zhiyuan Liu and Song Han and Maosong Sun},
  title        = {InfLLM: Unveiling the Intrinsic Capacity of LLMs for Understanding
                  Extremely Long Sequences with Training-Free Memory},
  journal      = {arXiv},
  year         = {2024}
}

@article{sun2024shadowkv,
  title={ShadowKV: KV Cache in Shadows for High-Throughput Long-Context LLM Inference},
  author={Sun, Hanshi and Chang, Li-Wen and Bao, Wenlei and Zheng, Size and Zheng, Ningxin and Liu, Xin and Dong, Harry and Chi, Yuejie and Chen, Beidi},
  journal={arXiv preprint arXiv:2410.21465},
  year={2024}
}

@misc{tang2024quest,
      title={Quest: Query-Aware Sparsity for Efficient Long-Context LLM Inference}, 
      author={Jiaming Tang and Yilong Zhao and Kan Zhu and Guangxuan Xiao and Baris Kasikci and Song Han},
      year={2024},
      eprint={2406.10774},
      archivePrefix={arXiv}
}

@misc{liu2024retrievalattentionacceleratinglongcontextllm,
      title={RetrievalAttention: Accelerating Long-Context LLM Inference via Vector Retrieval}, 
      author={Di Liu and Meng Chen and Baotong Lu and Huiqiang Jiang and Zhenhua Han and Qianxi Zhang and Qi Chen and Chengruidong Zhang and Bailu Ding and Kai Zhang and Chen Chen and Fan Yang and Yuqing Yang and Lili Qiu},
      year={2024},
      eprint={2409.10516},
      archivePrefix={arXiv},
      primaryClass={cs.LG},
      url={https://arxiv.org/abs/2409.10516}, 
}

@article{hooper2024squeezed,
  title={Squeezed Attention: Accelerating Long Context Length LLM Inference},
  author={Hooper, Coleman and Kim, Sehoon and Mohammadzadeh, Hiva and Maheswaran, Monishwaran and Paik, June and Mahoney, Michael W and Keutzer, Kurt and Gholami, Amir},
  journal={arXiv preprint arXiv:2411.09688},
  year={2024}
}

@misc{wu2024tokenselectefficientlongcontextinference,
      title={TokenSelect: Efficient Long-Context Inference and Length Extrapolation for LLMs via Dynamic Token-Level KV Cache Selection}, 
      author={Wei Wu and Zhuoshi Pan and Chao Wang and Liyi Chen and Yunchu Bai and Kun Fu and Zheng Wang and Hui Xiong},
      year={2024},
      eprint={2411.02886},
      archivePrefix={arXiv},
      primaryClass={cs.CL},
      url={https://arxiv.org/abs/2411.02886}, 
}

@misc{douze2024faisslibrary,
      title={The Faiss library}, 
      author={Matthijs Douze and Alexandr Guzhva and Chengqi Deng and Jeff Johnson and Gergely Szilvasy and Pierre-Emmanuel Mazaré and Maria Lomeli and Lucas Hosseini and Hervé Jégou},
      year={2024},
      eprint={2401.08281},
      archivePrefix={arXiv},
      primaryClass={cs.LG},
      url={https://arxiv.org/abs/2401.08281}, 
}

@misc{he2024fastdecodehighthroughputgpuefficientllm,
      title={FastDecode: High-Throughput GPU-Efficient LLM Serving using Heterogeneous Pipelines}, 
      author={Jiaao He and Jidong Zhai},
      year={2024},
      eprint={2403.11421},
      archivePrefix={arXiv},
      primaryClass={cs.DC},
      url={https://arxiv.org/abs/2403.11421}, 
}

@misc{su2023roformerenhancedtransformerrotary,
      title={RoFormer: Enhanced Transformer with Rotary Position Embedding}, 
      author={Jianlin Su and Yu Lu and Shengfeng Pan and Ahmed Murtadha and Bo Wen and Yunfeng Liu},
      year={2023},
      eprint={2104.09864},
      archivePrefix={arXiv},
      primaryClass={cs.CL},
      url={https://arxiv.org/abs/2104.09864}, 
}

@article{hsieh2024ruler,
  title={RULER: What's the Real Context Size of Your Long-Context Language Models?},
  author={Cheng-Ping Hsieh and Simeng Sun and Samuel Kriman and Shantanu Acharya and Dima Rekesh and Fei Jia and Yang Zhang and Boris Ginsburg},
  year={2024},
  journal={arXiv preprint arXiv:2404.06654},
}

@article{xiao2023streamingllm,
        title={Efficient Streaming Language Models with Attention Sinks},
        author={Xiao, Guangxuan and Tian, Yuandong and Chen, Beidi and Han, Song and Lewis, Mike},
        journal={arXiv},
        year={2023}
        }

@misc{xiong2024searchengineservicesmeet,
      title={When Search Engine Services meet Large Language Models: Visions and Challenges}, 
      author={Haoyi Xiong and Jiang Bian and Yuchen Li and Xuhong Li and Mengnan Du and Shuaiqiang Wang and Dawei Yin and Sumi Helal},
      year={2024},
      eprint={2407.00128},
      archivePrefix={arXiv},
      primaryClass={cs.IR},
      url={https://arxiv.org/abs/2407.00128}, 
}

@inproceedings{
zhang2023ho,
title={H2O: Heavy-Hitter Oracle for Efficient Generative Inference of Large Language Models},
author={Zhenyu Zhang and Ying Sheng and Tianyi Zhou and Tianlong Chen and Lianmin Zheng and Ruisi Cai and Zhao Song and Yuandong Tian and Christopher Re and Clark Barrett and Zhangyang Wang and Beidi Chen},
booktitle={Thirty-seventh Conference on Neural Information Processing Systems},
year={2023},
url={https://openreview.net/forum?id=RkRrPp7GKO}
}

@inproceedings{
li2024snapkv,
title={Snap{KV}: {LLM} Knows What You are Looking for Before Generation},
author={Yuhong Li and Yingbing Huang and Bowen Yang and Bharat Venkitesh and Acyr Locatelli and Hanchen Ye and Tianle Cai and Patrick Lewis and Deming Chen},
booktitle={The Thirty-eighth Annual Conference on Neural Information Processing Systems},
year={2024},
url={https://openreview.net/forum?id=poE54GOq2l}
}

@misc{huang2024locretenhancingevictionlongcontext,
      title={Locret: Enhancing Eviction in Long-Context LLM Inference with Trained Retaining Heads}, 
      author={Yuxiang Huang and Binhang Yuan and Xu Han and Chaojun Xiao and Zhiyuan Liu},
      year={2024},
      eprint={2410.01805},
      archivePrefix={arXiv},
      primaryClass={cs.CL},
      url={https://arxiv.org/abs/2410.01805}, 
}

@misc{cai2024pyramidkvdynamickvcache,
      title={PyramidKV: Dynamic KV Cache Compression based on Pyramidal Information Funneling}, 
      author={Zefan Cai and Yichi Zhang and Bofei Gao and Yuliang Liu and Tianyu Liu and Keming Lu and Wayne Xiong and Yue Dong and Baobao Chang and Junjie Hu and Wen Xiao},
      year={2024},
      eprint={2406.02069},
      archivePrefix={arXiv},
      primaryClass={cs.CL},
      url={https://arxiv.org/abs/2406.02069}, 
}

@misc{feng2024adakvoptimizingkvcache,
      title={Ada-KV: Optimizing KV Cache Eviction by Adaptive Budget Allocation for Efficient LLM Inference}, 
      author={Yuan Feng and Junlin Lv and Yukun Cao and Xike Xie and S. Kevin Zhou},
      year={2024},
      eprint={2407.11550},
      archivePrefix={arXiv},
      primaryClass={cs.CL},
      url={https://arxiv.org/abs/2407.11550}, 
}

@misc{li2024quickllamaqueryawareinferenceacceleration,
      title={QuickLLaMA: Query-aware Inference Acceleration for Large Language Models}, 
      author={Jingyao Li and Han Shi and Xin Jiang and Zhenguo Li and Hong Xu and Jiaya Jia},
      year={2024},
      eprint={2406.07528},
      archivePrefix={arXiv},
      primaryClass={cs.LG},
      url={https://arxiv.org/abs/2406.07528}, 
}

@article{roargraph,
author = {Chen, Meng and Zhang, Kai and He, Zhenying and Jing, Yinan and Wang, X. Sean},
title = {RoarGraph: A Projected Bipartite Graph for Efficient Cross-Modal Approximate Nearest Neighbor Search},
year = {2024},
issue_date = {July 2024},
publisher = {VLDB Endowment},
volume = {17},
number = {11},
issn = {2150-8097},
url = {https://doi.org/10.14778/3681954.3681959},
doi = {10.14778/3681954.3681959},
journal = {Proc. VLDB Endow.},
month = aug,
pages = {2735–2749},
numpages = {15}
}

@misc{ai2024yi,
    title={Yi: Open Foundation Models by 01.AI},
    author={01. AI and : and Alex Young and Bei Chen and Chao Li and Chengen Huang and Ge Zhang and Guanwei Zhang and Heng Li and Jiangcheng Zhu and Jianqun Chen and Jing Chang and Kaidong Yu and Peng Liu and Qiang Liu and Shawn Yue and Senbin Yang and Shiming Yang and Tao Yu and Wen Xie and Wenhao Huang and Xiaohui Hu and Xiaoyi Ren and Xinyao Niu and Pengcheng Nie and Yuchi Xu and Yudong Liu and Yue Wang and Yuxuan Cai and Zhenyu Gu and Zhiyuan Liu and Zonghong Dai},
    year={2024},
    eprint={2403.04652},
    archivePrefix={arXiv},
    primaryClass={cs.CL}
}

@misc{vllm,
      title={Efficient Memory Management for Large Language Model Serving with PagedAttention}, 
      author={Woosuk Kwon and Zhuohan Li and Siyuan Zhuang and Ying Sheng and Lianmin Zheng and Cody Hao Yu and Joseph E. Gonzalez and Hao Zhang and Ion Stoica},
      year={2023},
      eprint={2309.06180},
      archivePrefix={arXiv},
      primaryClass={cs.LG},
      url={https://arxiv.org/abs/2309.06180}, 
}

@article{minference,
    title={MInference 1.0: Accelerating Pre-filling for Long-Context LLMs via Dynamic Sparse Attention},
    author={Jiang, Huiqiang and Li, Yucheng and Zhang, Chengruidong and Wu, Qianhui and Luo, Xufang and Ahn, Surin and Han, Zhenhua and Abdi, Amir H and Li, Dongsheng and Lin, Chin-Yew and Yang, Yuqing and Qiu, Lili},
    journal={arXiv preprint arXiv:2407.02490},
    year={2024}
}

@online{needle,
  author = {Greg Kamradt},
  title = {Needle in a Haystack - Pressure Testing LLMs},
  year = {2023},
  url = {https://github.com/gkamradt/LLMTest_NeedleInAHaystack},
  urldate = {2024-08-12},
}

@misc{magicpig,
      title={MagicPIG: LSH Sampling for Efficient LLM Generation}, 
      author={Zhuoming Chen and Ranajoy Sadhukhan and Zihao Ye and Yang Zhou and Jianyu Zhang and Niklas Nolte and Yuandong Tian and Matthijs Douze and Leon Bottou and Zhihao Jia and Beidi Chen},
      year={2024},
      eprint={2410.16179},
      archivePrefix={arXiv},
      primaryClass={cs.CL},
      url={https://arxiv.org/abs/2410.16179}, 
}

@misc{gradientlongcontextllama3,
  title={Llama 3 Gradient: A series of long context models},
  author={Leonid Pekelis and Michael Feil and Forrest Moret and Mark Huang and Tiffany Peng},
  year={2024},
  url = {https://gradient.ai/blog/scaling-rotational-embeddings-for-long-context-language-models},
  doi = { 10.57967/hf/3372 },
}

@misc{gpa,
      title={GQA: Training Generalized Multi-Query Transformer Models from Multi-Head Checkpoints}, 
      author={Joshua Ainslie and James Lee-Thorp and Michiel de Jong and Yury Zemlyanskiy and Federico Lebrón and Sumit Sanghai},
      year={2023},
      eprint={2305.13245},
      archivePrefix={arXiv},
      primaryClass={cs.CL},
      url={https://arxiv.org/abs/2305.13245}, 
}

@String{Computing = "Computing" }

@ArtifactSoftware{R,
    title = {R: A Language and Environment for Statistical Computing},
    author = {{R Core Team}},
    organization = {R Foundation for Statistical Computing},
    address = {Vienna, Austria},
    year = {2019},
    url = {https://www.R-project.org/},
}

\clearpage
\appendix
\appendix \label{sec:appendix}
\section{Formal Proof of Obs.~\ref{obs_1}}  \label{sec:appendix_lemma}

Before proceeding with the formal proof, we first prove two auxiliary {Lemmas}~\ref{lem:lem1} and~\ref{lem:lem2}.

\begin{lemma} \label{lem:lem1}
    Let \( A = [a_1, a_2, \ldots, a_n] \) be an array of length \( n \), and let \( B = [b_1, b_2, \ldots, b_n] \) be a rearrangement of \( A \) such that the number of inversions \( t = \left| \{(i, j) \mid i < j \text{ and } a_i > a_j \text{ and } b_i < b_j\} \right| \); then, for any \( 1 \leq m \leq n \), the set \( S_m = \{ a_i \mid 1 \leq i \leq m \text{ and } a_i \notin \{b_1, b_2, \ldots, b_m\} \} \) satisfies \( |S_m| \leq \lfloor \sqrt{t} \rfloor \).
\end{lemma}

\begin{proof}
    Let \( S_m \) be the set of elements that are among the first \( m \) elements in \( A \) but are not among the first \( m \) elements in \( B \), and let \( s = |S_m| \). Similarly, there are \( s \) elements from the last \( n-m \) positions of \( A \) that are moved into the first \( m \) positions of \( B \). Denote these elements as the set \( T_m \).

    For each \( x \in S_m \) and \( y \in T_m \), if \( x > y \) in \( A \), then in \( B \), \( y \) is in the first \( m \) positions while \( x \) is in the last \( (n-m) \) positions, forming an inversion. Therefore, the total number of inversions \( t \) satisfies:
    \[
    t \geq \sum_{x \in S_m} \sum_{y \in T_m} \mathbf{1}_{\{x > y\}} \triangleq d,
    \]
    where \( \mathbf{1}_{\{x > y\}} \) is the indicator function (1 if \( x > y \), 0 otherwise).

    Since \( |S_m| = |T_m| = s \), there are \( s^2 \) pairs \( (x, y) \). To estimate \( d \), we observe that in the worst case, each \( x \in S_m \) must form at least one inversion with some \( y \in T_m \). This is because \( x \) and \( y \) are swapped between the first \( m \) and last \( (n-m) \) positions, and \( x > y \) must hold for at least one such pair to ensure \( x \) is not in the first \( m \) positions of \( B \). Thus, we have:
    \[
    d \geq s.
    \]

    However, a tighter bound can be derived by considering the total number of possible inversions. Since there are \( s^2 \) pairs \( (x, y) \), and each pair contributes at most one inversion, the maximum number of inversions is \( s^2 \). Therefore:
    \[
    t \geq d \geq s^2.
    \]

    Rearranging this inequality, we obtain:
    \[
    s \leq \sqrt{t}.
    \]

    Since \( s \) is an integer, it follows that:
    \[
    s \leq \lfloor \sqrt{t} \rfloor.
    \]

    Combining the above results, we conclude that \( |S_m| = s \leq \lfloor \sqrt{t} \rfloor \), completing the proof.
\end{proof}

\begin{lemma} \label{lem:lem2}
    Let \( K_1, K_2, Q_1, Q_2 \in \mathbb{R}^n \) be unit vectors, and assume \( Q_1 \cdot K_1 - Q_1 \cdot K_2 \triangleq \delta \ge 0 \) denote the difference in projections of \( Q_1 \) onto \( K_1 \) and \( K_2 \). Suppose the cosine similarity between \( Q_1 \) and \( Q_2 \) satisfies \( \cos(Q_1, Q_2) = \theta \), where \( \theta \in [0, \pi] \) is the angle between them. If \( \delta > 4 \cdot \sin\left(\frac{\theta}{2}\right) \), then it must satisfy 
    \[
    Q_2 \cdot K_1 - Q_2 \cdot K_2 \ge 0.
    \]
\end{lemma}

\begin{proof}
    We begin by noting that since \( Q_1 \) and \( Q_2 \) are unit vectors, the distance between them can be expressed as:
    \[
    \|Q_2 - Q_1\| = 2 \sin\left(\frac{\theta}{2}\right).
    \]
    
    Using the Cauchy-Schwarz inequality, we have:
    \[
    |Q_2 \cdot K_1 - Q_1 \cdot K_1| \le \|Q_2 - Q_1\| \cdot \|K_1\| = 2 \sin\left(\frac{\theta}{2}\right),
    \]
    and similarly,
    \[
    |Q_2 \cdot K_2 - Q_1 \cdot K_2| \le \|Q_2 - Q_1\| \cdot \|K_2\| = 2 \sin\left(\frac{\theta}{2}\right).
    \]
    
    From these inequalities, we deduce:
    \[
    Q_2 \cdot K_1 \ge Q_1 \cdot K_1 - 2 \sin\left(\frac{\theta}{2}\right),
    \]
    and
    \[
    Q_2 \cdot K_2 \le Q_1 \cdot K_2 + 2 \sin\left(\frac{\theta}{2}\right).
    \]
    
    Combining these results with the given condition \( Q_1 \cdot K_1 - Q_1 \cdot K_2 = \delta \), we obtain:
    \[
    Q_2 \cdot K_1 - Q_2 \cdot K_2 \ge \delta - 4 \sin\left(\frac{\theta}{2}\right).
    \]
    
    Since \( \delta > 4 \sin\left(\frac{\theta}{2}\right) \), it follows that:
    \[
    Q_2 \cdot K_1 - Q_2 \cdot K_2 \ge 0.
    \]
\end{proof}
\begin{theorem}
    Considering two vectors $Q, Q' \in \mathbb{R}^{d}$, with a search space of vector set $\{K_i\}_{i=1}^N$, a top-$k$ retrieval vector number $K$ and top-$K$ overlap percentage $p \in [0,1]$, there exist a threshold $\varepsilon$ such that when the cosine similarity of two vectors $cos(Q, Q') > \varepsilon$, the top-$K$ vector sets $S$ and $S'$ retrieved by V and V' would satisfy $\frac{|S\cap S'|}{|S|} > p$.
\end{theorem}

\begin{proof}
Let the search space \(\{K_i\}_{i=1}^N\) be sorted in descending order of similarity to \( Q \) as array \( A = [a_1, a_2, \ldots, a_N] \), where \( a_i = \cos(Q, K_i) \). Similarly, sort the vectors by similarity to \( Q' \) as array \( B = [b_1, b_2, \ldots, b_N] \), where \( b_i = \cos(Q', K_i) \). Define the number of inversions \( t \) as:
\[
t = \left| \{(i, j) \mid i < j \text{ and } a_i > a_j \text{ and } b_i < b_j\} \right|,
\]
which measures the difference between the two sorted arrays \( A \) and \( B \). By {Lemma} \ref{lem:lem1}, for any \( 1 \leq m \leq N \), the set \( S_m = \{ a_i \mid 1 \leq i \leq m \text{ and } a_i \notin \{b_1, b_2, \ldots, b_m\} \} \) satisfies:
\[
|S_m| \leq \lfloor \sqrt{t} \rfloor.
\]
Let \( m = K \). Then \( S_K \) represents the set of elements in the top-\( K \) of \( A \) that are not in the top-\( K \) of \( B \), and its size satisfies \( |S_K| \leq \sqrt{t} \). Therefore, the size of the intersection \( S \cap S' \) is:
\[
|S \cap S'| = K - |S_K| \geq K - \sqrt{t},
\]
and the overlap ratio is:
\[
\frac{|S \cap S'|}{K} \geq 1 - \frac{\sqrt{t}}{K}.
\]
To ensure \( \frac{|S \cap S'|}{K} > p \), we require:
\[
\begin{aligned}
1 - \frac{\sqrt{t}}{K} &> p \\
\implies \sqrt{t} &< K(1 - p) \\
\implies t &< K^2(1 - p)^2.
\end{aligned}
\]

Next, we use {Lemma}~\ref{lem:lem2} to control the number of inversions \( t \). Define the set of pairwise similarity differences in \( A \) as:
\[
\Delta_A = \{ d \mid d = a_i - a_j \text{ for } a_i, a_j \in A \text{ and } i < j \}.
\]
Denoting the \( t \)-th smallest element in \( \Delta_A \) as $d_{(t)}$, we set $\theta = 2\arcsin\frac{d_{(t)}}{4}$, which means $4\sin\frac{\theta}{2}=d_{(t)}$. According to {Lemma} \ref{lem:lem2}, there would be at most $t$ inversions.

Therefore, we could set $\theta = 2\arcsin\frac{d_{(K^2(1-p)^2)}}{4}$ for guarantee that $\frac{|S\cap S'|}{|S|} > p$.
\end{proof}

\section{Additional Comparative analysis with \model{} and other methods}
% \subsection{Comparison of \model{} and ANN methods in Index Construction Time} \label{sec:appendix_ann_prefill}

% We compare the index construction time of \model{} with three approximate nearest neighbor (ANN) methods, \textsc{IVF}, \textsc{HNSW} and \textsc{KMeans}, implemented in the Faiss library.
% The evaluation is conducted across varying lengths of Key vectors, ranging from 4K to 128K.

% As shown in \autoref{fig:index_construct_speed}, \model{} separately achieves up to a $50\times$, $1000\times$ and $10000\times$ reduction in index construction time compared to \textsc{HNSW}, \textsc{IVF} and \textsc{KMeans} methods.
% Meanwhile, \model{} demonstrates significantly greater stability, with smaller fluctuations in construction time as the context length increases, further showcasing its scalability and efficiency for long-context scenarios.

% \begin{figure}[!htbp]
%     \centering
%     \includegraphics[width=\linewidth]{figure/A_prefill_speed.pdf}
%     \caption{Index construction time of \model{} compared to ANN methods \textsc{HNSW} and \textsc{IVF} across varying Key vector lengths.}
%     \label{fig:index_construct_speed}
% \end{figure}

\subsection{Comparison of \textsc{MagicPIG} and \model{} within same sparsity ratio} \label{sec:dyn_spar}
\textsc{MagicPIG} uses an LSH-based sampling approach, and unlike top-$k$ methods, its sparsity budget cannot be precisely fixed. Experimentally, we adopted the recommended setting from the paper of \textsc{MagicPIG}: 9 hash functions and 120 hash tables. As shown in Table 8 in the \textsc{MagicPIG} paper, this yields a sparsity budget of around 4\%, which exceeds 512 tokens in the vast majority of benchmarks we evaluate. As expected, reducing the budget to 512 would further degrade performance of \textsc{MagicPIG}. Moreover, to enable a fairer comparison, we also evaluated \model{} under the same 4\% sparsity ratio in \textsf{Ruler} benchmark on \textsf{Llama-3-8B}. As shown in \autoref{tb:dyn_spar}, \textsc{\model{}} with dynamic 4\% sparsity budget achieves higher accuracy than both MagicPIG with 4\% sparsity ratio and \textsc{\model{}} with a fixed budget of 512.

\begin{table}[!htbp]
\centering
\footnotesize
\begin{tblr}{
  colsep=2pt, rowsep=0.8pt,
  column{2-7} = {c},
  hline{1,2,3,6} = {-}{},
  vline{2,7} = {-}{},
  row{2} = {bg=mygray},
  row{6} = {bg=lightblue},
}
\textbf{Context Len.} & \textbf{8K} & \textbf{16K} & \textbf{32K} & \textbf{64K} & \textbf{128K} & \textbf{AVG.} \\
\textsc{FullKV} & 90.97 & 90.10 & 86.17 & 83.06 & 79.65 & 85.99 \\
\textsc{MagicPIG (dyn.4\%)} & 85.04 & 85.00 & 80.34 & 75.27 & 67.51 & 78.63 \\
{\model{}$_{512}$ (fixed)} & \textbf{89.90} & 89.65 & \textbf{86.42} & 82.71 & 74.82 & 84.70 \\
{\model{} \textsc{(dyn.4\%)}} & 89.84 & \textbf{89.80} & 86.21 & \textbf{83.12} & \textbf{78.93} & \textbf{85.58} \\
\end{tblr}
\caption{Accuracy comparison of \model{} and \textsc{MagicPIG} with 4\% sparsity ratio on \texttt{Llama-3-8B-262K}.}
\label{tb:dyn_spar}
\end{table}

\subsection{Comprehensive Throughput and Accuracy comparison of \model{} and \textsc{MagicPIG} under varying parameter settings} \label{sec:compre_compare}
To provide a comprehensive comparison between MagicPIG and \model{} in terms of both accuracy and throughput, we selected three configurations in hyperparameters of MagicPIG, which was all tested and recommended in the paper of MagicPIG: (8 hash functions, 75 hash tables), (9 hash functions, 120 hash tables), and (10 hash functions, 150 hash tables). These configurations represent a trade-off spectrum, ranging from higher accuracy with lower throughput to lower accuracy with higher throughput. For \model{}, we used the same hyperparameter settings as reported in \autoref{fig:centroid}.

We evaluated accuracy of these methods on the \textsf{Ruler} benchmark with \textsf{LLaMA-3-8B} and throughput across different batch sizes under a 96k context length setting. As shown in \autoref{tb:magicpig_ablation1} and \autoref{fig:magicpig_ablation2}, \model{} outperforms the most accurate configuration (8, 75) of MagicPIG in terms of accuracy, while also achieving higher throughput than the most efficient (10, 150) MagicPIG configuration. These results demonstrate the superiority of \model{} in balancing both accuracy and efficiency.

\begin{table}[!htbp]
\centering
\footnotesize
\begin{tblr}{
  colsep=2pt, rowsep=0.8pt,
  column{2-7} = {c},
  hline{1,2,3,6,7} = {-}{},
  vline{2,7} = {-}{},
  row{2} = {bg=mygray},
  row{6} = {bg=lightblue},
}
\textbf{Method} & \textbf{8K} & \textbf{16K} & \textbf{32K} & \textbf{64K} & \textbf{128K} & \textbf{Mean} \\
\textbf{FullKV} & 90.97 & 90.10 & 86.17 & 83.06 & 79.65 & 85.99 \\
\textsc{MagicPIG}(10, 150) & 73.34 & 79.31 & 79.21 & 75.10 & 70.37 & 62.69 \\
\textsc{MagicPIG}(9, 120) & 81.55 & 87.97 & 87.18 & 82.75 & 78.91 & 70.93 \\
\textsc{MagicPIG}(8, 75) & 85.04 & 85.00 & 80.34 & 75.27 & 67.51 & 78.63 \\
\model{}$_{512}$ & 89.90 & 89.65 & 86.42 & 82.71 & 74.82 & \textbf{84.70} \\
\end{tblr}
\caption{On \texttt{Llama-3-8B} in \textsf{RULER} benchmark, \model{} achieves higher accuracy comparing to the most accurate configuration of \textsc{MagicPIG}.}
\label{tb:magicpig_ablation1}
\end{table}

\begin{figure}[!htbp]
    \centering
    \includegraphics[width=\linewidth]{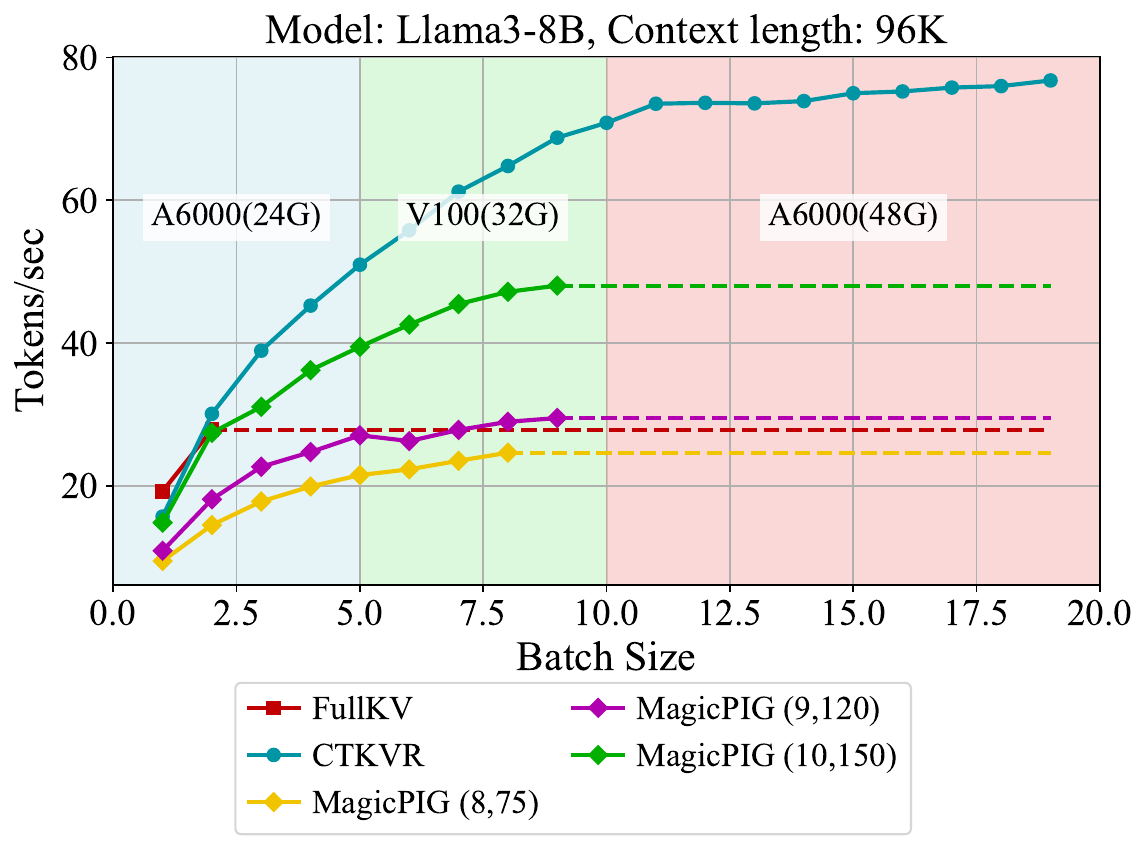}
    \caption{On \textsf{Llama-3-8B} \model{} with $\rho'=512$ achieves throughtput boost of $1.5\times$ comparing to most efficient configuration of \textsc{MagicPIG}.}
    \label{fig:magicpig_ablation2}
\end{figure}

\subsection{Comparison of \model{} and \textsc{RetrievalAttention}} \label{sec:retrievalattn}

While both \model{} and \textsc{RetrievalAttention} leverage query similarity for attention computation, \model{} more effectively exploits adjacent-query similarity through attention-specific optimizations. Built on a fundamentally different clustering algorithm, \model{} addresses the inefficiencies of \textsc{RetrievalAttention} in both the prefilling and decoding stages, while maintaining comparable accuracy.

Notably, \textsc{RetrievalAttention} depends on the \textsc{RoarGraph} algorithm for index construction and search, which is computationally intensive. As shown in \autoref{tb:retrievalattn2}, \textsc{RetrievalAttention} incurs up to $10000\times$ index construction time than \model{} across a range of context lengths, making it less practical for real-time inference. Moreover, as demonstrated in \autoref{tb:retrievalattn1}, \model{} achieves similar accuracy to \textsc{RetrievalAttention} on the \textsf{RULER} benchmark using \textsf{LLaMA-3-8B}, under the same sparsity budget (1024).

\begin{table}[!htbp]
\centering
\footnotesize
\begin{tblr}{
  colsep=2pt, rowsep=0.8pt,
  column{2-7} = {c},
  hline{1,2,3,4,5} = {-}{},
  vline{2,7} = {-}{},
  row{2} = {bg=mygray},
  row{4} = {bg=lightblue},
}
\textbf{Method} & \textbf{8K} & \textbf{16K} & \textbf{32K} & \textbf{64K} & \textbf{128K} & \textbf{Mean} \\
\textsc{FullKV} & 90.97 & 90.10 & 86.17 & 83.06 & 79.65 & 85.99 \\
\textsc{RetrievalAttn}$_{1024}$ & 90.10 & 89.92 & \textbf{86.42} & 82.96 & \textbf{77.01} & \textbf{85.28} \\
\model{}$_{1024}$ & \textbf{90.14} & \textbf{89.93} & 86.20 & 83.38 & \textbf{76.60} & 85.25 \\
\end{tblr}
\caption{Accuracy comparison of \model{}, \textsc{RetrievalAttention} on \textsf{RULER} across varying context lengths on \texttt{Llama-3-8B}.}
\label{tb:retrievalattn1}
\end{table}

\begin{table}[!htbp]
\centering
\footnotesize
\begin{tblr}{
  colsep=2pt, rowsep=0.8pt,
  column{2-7} = {c},
  hline{1,2,3,4} = {-}{},
  vline{2} = {-}{},
}
\textbf{Method} & \textbf{4K} & \textbf{8K} & \textbf{16K} & \textbf{32K} & \textbf{64K} & \textbf{128K} \\
\textsc{RetrievalAttn}$_{1024}$ & 140 & 289 & 587 & 1225 & 2463 & 4247 \\
\model{}$_{1024}$ & 0.21 & 0.22 & 0.23 & 0.25 & 0.29 & 0.43 \\
\end{tblr}
\caption{Index construction time(ms) of \model{} compared to \textsc{RetrievalAttention} across varing Key vector lengths.}
\label{tb:retrievalattn2}
\end{table}

\subsection{Comparison of \model{} and \textsc{SqueezedAttention}} \label{sec:squeezeattn}
Though \model{} and concurrent work \textsc{SqueezedAttention} share similar insight of two-stage retrieval, \model{} makes better use of adjacent-query similarity features in attention-specific optimization. With a completely different clustering algorithm, \model{} addresses limitations in accuracy and prefilling efficiency of \textsc{SqueezedAttention}:
\subsubsection{Inaccuracy for not considering key magnitudes}
The goal of top-$k$ retrieval is to find keys with the highest $QK^T$ scores for a given query. Thus, an effective clustering strategy should group keys with similar $QK^T$ scores, allowing irrelevant keys to be filtered out while preserving important keys.

\textsc{SqueezedAttention} clusters keys using cosine similarity, which overlooks magnitude information of the key crucial to $QK^T$ scores. This may group dissimilar keys together and lead to the loss of important context during retrieval. In contrast, \model{} clusters keys based on their relevance to a representative query, preserving high $QK^T$ scores within each cluster. As shown in \textbf{Obs 2} and \textbf{Theorem 1}, queries similar to a cluster’s representative query tend to retrieve keys with high $QK^T$ scores, enhancing the retention of relevant information.

Experimentally, we evaluate \textsc{SqueezedAttention} on the \textsf{RULER} benchmarks with a 32K context length. To create a more challenging retrieval setting than the paper of \textsc{SqueezedAttention} with sparse budget of 3k, we reduce the budget to 512. As shown in \autoref{tb:squeeze_comparison}, \model{} outperforms \textsc{SqueezedAttention} across all datasets, demonstrating stronger retrieval capability.

\begin{table*}[!htbp]
\centering
\footnotesize
\begin{tblr}{
  colsep=2pt, rowsep=0.8pt,
  column{2-15} = {c},
  hline{1,2,3,4,5} = {-}{},
  vline{2,15} = {-}{},
  row{2} = {bg=mygray},
  row{4} = {bg=lightblue},
}
\textbf{Method} & \textbf{single1} & \textbf{single2} & \textbf{single3} & \textbf{Mkey1} & \textbf{Mkey2} & \textbf{Mquery} & \textbf{Mvalue} & \textbf{qa1} & \textbf{qa2} & \textbf{fwe} & \textbf{vt} & \textbf{cwe} & \textbf{Mkey3} & \textbf{Mean} \\
\textsc{FullKV} & 100.00 & 100.00 & 100.00 & 100.00 & 100.00 & 100.00 & 93.48 & 80.20 & 60.41 & 90.62 & 95.41 & 3.22 & 96.87 & 86.17 \\
\textsc{SqueezedAttention} & 100.00 & 57.29 & 41.20 & 40.63 & 32.29 & 38.92 & 10.68 & 79.17 & 56.25 & 30.56 & 13.95 & 20.45 & 28.90 & 42.33 \\
\model{}$_{512}$ & 100.00 & 100.00 & 100.00 & 100.00 & 100.00 & 100.00 & 90.88 & 80.20 & 58.33 & 80.90 & 95.20 & 21.04 & 96.88 & 86.42 \\
\end{tblr}
\caption{Accuracy comparison of \model{} and \textsc{SqueezedAttention} under \textsf{RULER} benchmark on \texttt{Llama-3-8B}.}
\label{tb:squeeze_comparison}
\end{table*}

\subsubsection{Prefilling inefficiency for slow KMeans}
An analysis of the \textsc{SqueezedAttention} codebase reveals that its clustering is performed online for each input context, rather than generating reusable offline clusters for downstream tasks. Moreover, it relies on the KMeans algorithm for cluster construction, an expensive iterative process (up to 300 iterations in practice). In contrast, \model{} uses a single-pass $QK^T$-based assignment, optimized for GPU execution, resulting in a much faster and more efficient clustering process. As shown in \autoref{fig:index_construct_speed}, \model{} achieves up to $10000\times$ speedup across various context lengths compared to KMeans in \textsc{SqueezedAttention}.

\section{Additional Parameter and Component Experiments on \model{}}
\subsection{Model Architecture}
\autoref{tb:model_info} compares the architectural parameters differences of two models \texttt{Llama-3-8B-262K} and \texttt{Yi-9B-200K} used in our experiments. All models implement the \emph{Grouped-Query Attention} (GQA), where multiple query heads share one single Key and Value head for reducing storage overhead of the KV Cache.

\begin{table}[!htbp]
\centering
\caption{Architectural parameters of used LLMs.}
\label{tb:model_info}
\footnotesize
\begin{tblr}{
  vline{2} = {-}{},
  hline{1-2,4} = {-}{},
  colsep=4pt, rowsep=1.5pt,
  column{2-8} = {c}, 
}
                & \textbf{Layers} & \textbf{Query Head} & \textbf{KV Head}  \\
\texttt{Llama-3-8B-262K} & 32     & 32         & 8               \\
\texttt{Yi-9B-200K}      & 48     & 32         & 4                
\end{tblr}
\end{table}

\subsection{Deduplication Coefficient \mbox{$\alpha$} under Different \mbox{$\rho$, $C$ and $C'$} Values} \label{sec:appendix_alpha}

We evaluate the stability of the deduplication coefficient $\alpha$ by analyzing its behavior under varying values of $\rho$, $C$, and $C'$.
For this experiment, only $\rho$, $C$, and $C'$ are varied, with all other variables fixed at their default values, namely $\rho = 2560$, $C = 1024$, and $C' = 4$. 

As shown in \autoref{fig:alpha_abla}, $\alpha$ remains consistently below $0.5$ across the majority of tested configurations, further reinforcing the importance of the deduplication mechanism.

\begin{figure*}[htbp]
    \centering
    \includegraphics[width=\linewidth]{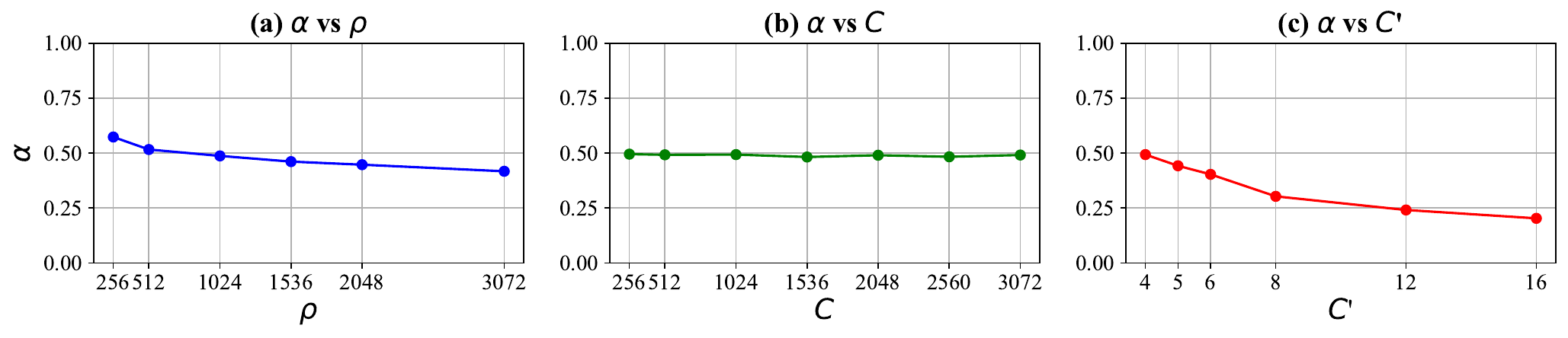}
    \caption{Impact of $\rho$, $C$, $C'$ on deduplication coefficient $\alpha$.}
    \label{fig:alpha_abla}
\end{figure*}

% \subsection{Scaling up to larger models} \label{sec:appendix_large}

\subsection{Additional Result on \textsf{Needle-in-a-haystack}} \label{sec:appendix_needle}

\autoref{fig:full_needle} presents the results of all methods on the \textsf{Needle-in-a-haystack} dataset. 

Notably, \textsc{Snap} encounters retrieval failures as the document length increases. While methods such as \textsc{ShadowKV}, \textsc{Flat}, \textsc{MagicPIG}, and \textsc{Quest} demonstrate strong performance on this dataset, their degradation in both accuracy and speed has already been analyzed in the context of more complex datasets, such as \textsf{RULER}, in previous evaluations.

\begin{figure*}[!htbp]
    \centering
    \includegraphics[width=\linewidth]{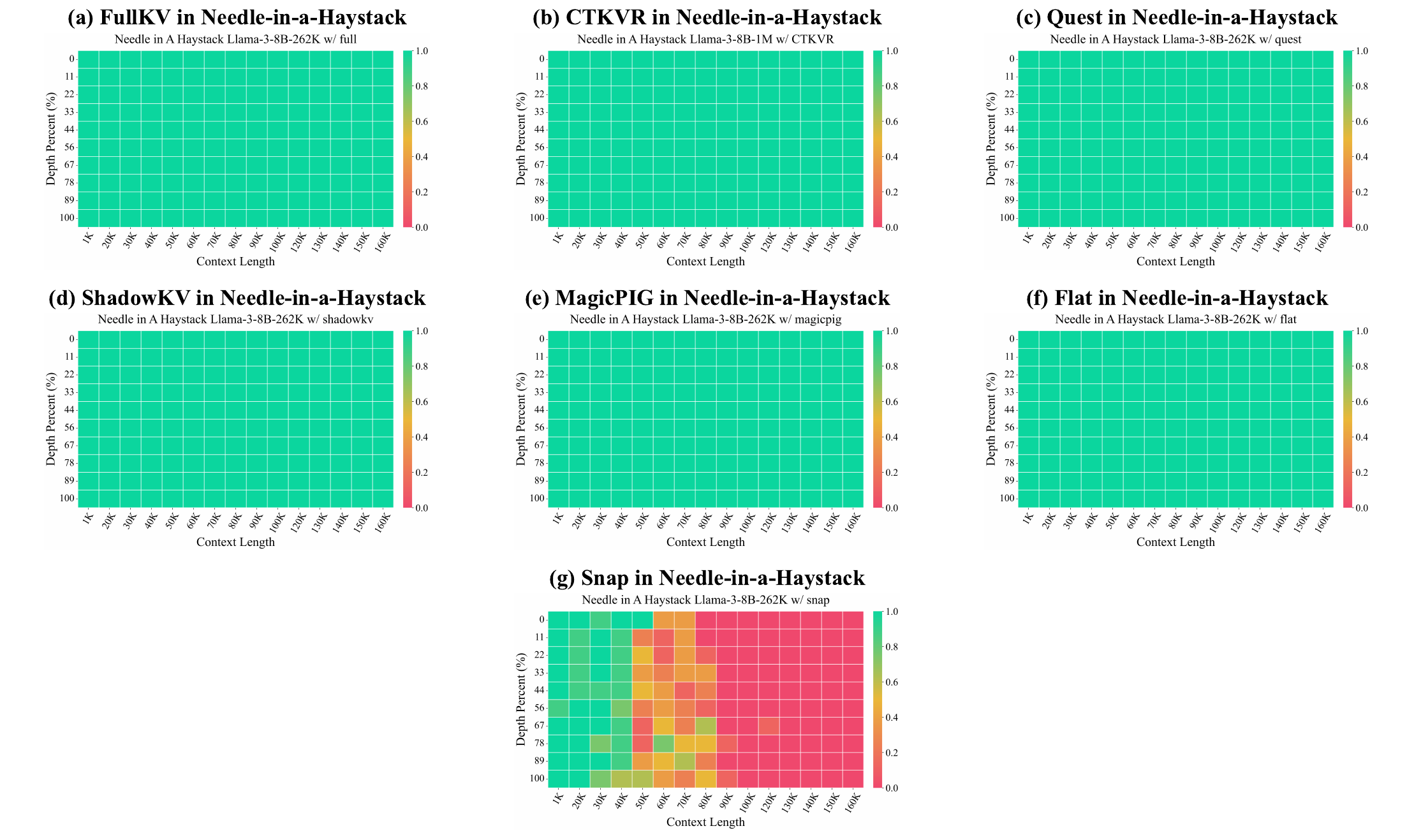}
    \caption{Performance of different methods on \textsf{Needle-in-a-haystack}.}
    \label{fig:full_needle}
\end{figure*}

% \subsection{Scaling up to Extremely Long-Context Inference} \label{sec:appendix_longlong}

% We evaluate \model{} on the \textsf{Needle-in-a-haystack} dataset with extremely long contexts ranging from 200K to 1 million (1M) tokens, using \texttt{Llama-3-8B-1M} \citep{gradientlongcontextllama3}. 

% As shown in \autoref{fig:longlong}, \model{} successfully retrieves all needles, demonstrating the robustness and effectiveness of our indexing methods in handling ultra-long contexts.

% \begin{figure}[!htbp]
%     \centering
%     \includegraphics[width=\linewidth]{figure/A_needle_CTKVR_longlong.pdf}
%     \caption{Performance of \model{} on \textsf{Needle-in-a-haystack} with context lengths ranging from 200K to 1M tokens, evaluated on \texttt{Llama-3-8B-1M}.}
%     \label{fig:longlong}
% \end{figure}

\subsection{Full Score Results on \textsf{LongBench}} \label{sec:appendix_longbench}

In Section \ref{sec:acc}, we reported category-level scores for \textsf{LongBench}, which were computed as the average performance across all tasks within each category. In this section, we provide a detailed breakdown of \model{}'s performance on individual tasks in \textsf{LongBench}.

As shown in \autoref{tb:full_longbench}, \model{} demonstrates consistently strong performance across all tasks, maintaining stable accuracy within 1\% of \textsc{FullKV} attention or even surpassing it in specific tasks such as \textit{Qasper} and \textit{LCC}.
These results further highlight the robustness and effectiveness of \model{} in handling diverse long-context scenarios.

\begin{table*}[!htbp]
\centering
\footnotesize
\begin{tblr}{
    colspec = {c c c c c c c c c c c c c c c c}, % 14 列
    cell{3}{1} = {r=8}{},
    cell{11}{1} = {r=8}{},
    cell{1}{2} = {r=2}{},
    cell{1}{16} = {r=2}{},
    cell{1}{3} = {c=3}{},
    cell{1}{6} = {c=3}{},
    cell{1}{9} = {c=3}{},
    cell{1}{12} = {c=2}{},
    cell{1}{14} = {c=2}{},
    row{1} = {font=\bfseries},
  % vline{2-3,8} = {-}{},
  vline{2,3,6,9,12,14,16} = {-}{},
  hline{1,3,4,9,11,12,17,19} = {-}{},
    row{9,10,17,18} = {bg=lightblue},
    row{3,11} = {bg=mygray},
  cell{3,11}{1} = {bg=white},
    colsep=1.3pt,
    rowsep=1pt,
}
& Method & \textbf{S-Doc} & & & \textbf{M-Doc} & & & \textbf{Summ} & & & \textbf{Few-Shot} & & \textbf{Code} &  & AVG. \\
&  & \textit{NQA} & \textit{Multi\_en} & \textit{Qasper} & \textit{HQA} & \textit{Musique} & \textit{2wiki} & \textit{Multinews} & \textit{Qmsum} & \textit{Gov} & \textit{Trec} & \textit{TQA} & \textit{Lcc} & \textit{Repobench} &  \\
\rotatebox{90}{\texttt{Llama-3-8B-262K}} & \textsc{FullKV} & 16.46 & 43.41 & 27.53 & 28.08 & 14.94 & 22.44 & 27.97 & 25.51 & 34.50 & 70.50 & 87.57 & 52.00 & 46.99 & 38.30 \\
    & \textsc{Snap} & 1.51 & 11.00 & 3.75 & 7.49 & 3.64 & 7.62 & 18.28 & 7.23 & 11.55 & 47.25 & 32.98 & 34.71 & 19.40 & 15.88 \\
    & \textsc{MagicPIG} & 9.89 & 28.87 & 15.74 & 11.16 & 6.08 & 11.53 & 27.19 & 20.08 & 33.10 & 69.50 & 86.66 & 52.03 & 45.60 & 32.11 \\
    & \textsc{Quest} & 15.97 & 47.31 & 29.00 & 27.28 & 14.11 & 18.51 & 22.10 & 24.96 & 34.06 & 68.67 & 89.00 & 53.25 & 47.26 & 37.81 \\
    & \textsc{ShadowKV} & 15.81 & 49.34 & 27.04 & 26.45 & 13.63 & 13.57 & 32.56 & 25.64 & 33.94 & 70.50 & 87.40 & 52.10 & 46.20 & 38.01 \\
    & \textsc{Flat} & 16.07 & 44.88 & 28.88 & 35.09 & 14.20 & 22.74 & 27.14 & 25.48 & 34.15 & 70.50 & 86.20 & 53.01 & 47.50 & 38.91 \\
    & \textsc{\model{}}$_{512}$ & 15.74 & 44.21 & 28.70 & 25.32 & 14.23 & 24.21 & 28.02 & 25.58 & 34.35 & 70.00 & 86.24 & 53.55 & 47.27 & \textbf{38.26} \\
    & \textsc{\model{}}$_{1024}$ & 16.06 & 44.71 & 28.27 & 25.94 & 14.16 & 23.49 & 27.68 & 25.20 & 34.05 & 70.00 & 86.24 & 53.09 & 47.48 & 38.18 \\ \hline
\rotatebox{90}{\texttt{Llama-3-8B-262K}} & \textsc{FullKV} & 12.12 & 33.76 & 38.23 & 51.84 & 27.99 & 35.80 & 26.71 & 20.27 & 30.59 & 77.00 & 90.00 & 72.26 & 67.35 & 44.92 \\
    & \textsc{Snap} & 4.53 & 11.51 & 10.44 & 11.38 & 6.51 & 12.50 & 20.31 & 6.43 & 7.14 & 62.25 & 42.51 & 29.59 & 20.73 & 18.91 \\
    & \textsc{MagicPIG} & 10.68 & 32.67 & 36.67 & 51.33 & 27.35 & 35.96 & 23.57 & 20.55 & 30.11 & 77.00 & 90.32 & 71.97 & 66.31 & 44.19 \\
    & \textsc{Quest} & 13.01 & 33.04 & 38.52 & 50.76 & 26.98 & 34.42 & 28.56 & 20.80 & 28.55 & 77.00 & 89.67 & 71.43 & 65.49 & 44.48 \\
    & \textsc{ShadowKV} & 12.00 & 39.11 & 38.58 & 51.60 & 27.24 & 12.13 & 30.95 & 21.22 & 30.97 & 77.00 & 90.22 & 72.70 & 65.91 & 43.82 \\
    & \textsc{Flat} & 14.04 & 34.58 & 38.44 & 50.62 & 27.16 & 36.18 & 26.43 & 21.67 & 30.76 & 77.00 & 90.22 & 70.93 & 67.31 & 45.03 \\
    & \textsc{\model{}}$_{512}$ & 13.40 & 33.88 & 38.89 & 51.12 & 26.92 & 36.43 & 26.76 & 21.50 & 30.20 & 77.00 & 90.22 & 72.31 & 67.13 & \textbf{45.06} \\
    & \textsc{\model{}}$_{1024}$ & 13.91 & 33.67 & 38.74 & 51.37 & 27.19 & 35.68 & 27.04 & 21.47 & 30.21 & 77.00 & 90.22 & 71.43 & 67.27 & 45.02 \\
\end{tblr}
\caption{Accuracy comparison of different methods across each tasks in \textsf{LongBench}.}
\label{tb:full_longbench}
\end{table*}

\subsection{Computational Cost and Accuracy Tradeoff in \model{}} \label{sec:tradeoff}
As shown in \autoref{tb:runtime}, we benchmark the runtime of LLaMA-3-8B on context length of 96k and test computation overhead for each component . We find that \model{} spends most time on CPU operations, especially in $Q2K$ Rerank and attention computation. This indicates that the main bottleneck lies in CPU components rather than clustering locating.

\begin{table*}[!htbp]
\centering
\footnotesize
\begin{tblr}{
  colsep=2pt, rowsep=0.8pt,
  column{2-5} = {c},
  hline{1,2,3} = {-}{},
  vline{2} = {-}{},
}
\textbf{Operation} & \textbf{$Q2Q^{(c)}$ Recall} & \textbf{$Q2K$ Rerank} & \textbf{CPU Attention} & \textbf{GPU Attention} \\
\textbf{Time (ms)} & 5.57 & 35.68 & 15.80 & 2.84 \\
\end{tblr}
\caption{Runtime of different module in \model{}.}
\label{tb:runtime}
\end{table*}

To further explore the trade-off between computation and accuracy under different indexing settings, we evaluate \model{} on the \textsf{ruler\/multi\_key\_3} dataset with sparse budget of 512 and 96k context length, varying the number of maintained centroids $C$ and retrieved centroid $C'$. As shown in \autoref{tb:trade1} and \autoref{tb:trade2}, increasing either $C$ or $C'$ reduces throughput, as $C$ affects $Q2Q^{(c)}$ Recall time and $C'$ impacts $Q2K$ Rerank time. Throughput drops more sharply with higher $C'$, since the CPU is the system bottleneck. Meanwhile, accuracy increases with higher values of $C$ and $C'$, stabilizing near the full KNN result at $C=320$ and $C'=5$.

\begin{table*}[!htbp]
\centering
\footnotesize
\begin{tblr}{
  colsep=2pt, rowsep=0.8pt,
  column{2-9} = {c},
  hline{1,2,3,4} = {-}{},
  vline{2,9} = {-}{},
}
\textbf{Centroid Num $C$} & \textbf{20} & \textbf{40} & \textbf{80} & \textbf{160} & \textbf{320} & \textbf{640} & \textbf{1280} & \textbf{1920} \\
\textbf{Throughput (Tokens/s)} & 16.76 & 16.71 & 16.67 & 16.51 & 16.47 & 16.43 & 16.40 & 16.35 \\
\textbf{Accuracy (Acc)} & 0.4167 & 0.3020 & 0.4791 & 0.6250 & 0.8020 & 0.8645 & 0.8645 & 0.8645 \\
\end{tblr}
\caption{Impact of the number of selected centroids $C$ on throughput and accuracy.}
\label{tb:trade1}
\end{table*}

\begin{table*}[!htbp]
\centering
\footnotesize
\begin{tblr}{
  colsep=2pt, rowsep=0.8pt,
  column{2-8} = {c},
  hline{1,2,3,4} = {-}{},
  vline{2,8} = {-}{}
}
\textbf{Retrieved Centroid Num $C'$} & \textbf{2} & \textbf{4} & \textbf{6} & \textbf{8} & \textbf{10} & \textbf{12} & \textbf{14} \\
\textbf{Throughput (Tokens/s)} & 17.47 & 16.41 & 15.46 & 14.25 & 12.89 & 11.47 & 10.14 \\
\textbf{Accuracy (Acc)} & 0.8854 & 0.8958 & 0.8988 & 0.8990 & 0.8990 & 0.8990 & 0.8990 \\
\end{tblr}
\caption{Effect of the number of retrieved centroids $C'$ on throughput and accuracy.}
\label{tb:trade2}
\end{table*}

\end{document}